\documentclass{article}

\usepackage{microtype}
\usepackage{graphicx}
\usepackage{subfigure}
\usepackage{booktabs} % for professional tables
\usepackage{multirow}
\usepackage{enumitem}
\usepackage[most]{tcolorbox}
\usepackage{bbm}
\usepackage{hyperref}

\usepackage[preprint]{icml2026}
\usepackage{amsmath}
\usepackage{amssymb}
\usepackage{mathtools}
\usepackage{amsthm}
\DeclareMathOperator*{\argmax}{arg\,max}
\usepackage[capitalize,noabbrev]{cleveref}

\theoremstyle{plain}
\newtheorem{theorem}{Theorem}[section]
\newtheorem{proposition}[theorem]{Proposition}

\newtheorem{corollary}[theorem]{Corollary}
\theoremstyle{definition}
\newtheorem{definition}[theorem]{Definition}
\newtheorem{assumption}[theorem]{Assumption}
\theoremstyle{remark}
\newtheorem{remark}[theorem]{Remark}
\newtheorem*{theorem*}{Theorem}
\newtheorem*{proposition*}{Proposition}
\newtheorem*{lemma*}{Lemma}
\newtheorem*{corollary*}{Corollary}

\usepackage[textsize=tiny]{todonotes}

\icmltitlerunning{Why Linear Interpretability Works: Invariant Subspaces as a Result of Architectural Constraints}

\begin{document}

\twocolumn[
\icmltitle{Why Linear Interpretability Works: Invariant Subspaces as a Result of Architectural Constraints}




\begin{icmlauthorlist}
\icmlauthor{Andres Saurez}{aff1}
\icmlauthor{Yousung Lee}{aff1}
\icmlauthor{Dongsoo Har}{aff1}
\end{icmlauthorlist}

\icmlaffiliation{aff1}{Korea Advanced Institute of Science and Technology, Daejeon 34051, South Korea}

\icmlcorrespondingauthor{Dongsoo Har}{dshar@kaist.ac.kr}

\icmlkeywords{Machine Learning, ICML}

\vskip 0.3in

]
% \printAffiliationsAndNotice{\icmlPreprintNotice}
\printAffiliationsAndNotice{}

\begin{abstract}
Linear probes and sparse autoencoders consistently recover meaningful structure from transformer representations---yet why should such simple methods succeed in deep, nonlinear systems? We show this is not merely an empirical regularity but a consequence of architectural necessity: transformers communicate information through linear interfaces (attention OV circuits, unembedding matrices), and any semantic feature decoded through such an interface must occupy a context-invariant linear subspace. We formalize this as the \emph{Invariant Subspace Necessity} theorem and derive the \emph{Self-Reference Property}: tokens directly provide the geometric direction for their associated features, enabling zero-shot identification of semantic structure without labeled data or learned probes. Empirical validation in eight classification tasks and four model families confirms the alignment between class tokens and semantically related instances. Our framework provides \textbf{a principled architectural explanation} for why linear interpretability methods work, unifying linear probes and sparse autoencoders.
\end{abstract}



\section{Introduction}

Linear structure has become a central organizing principle in modern interpretability for transformers \citep{vaswani2017attention}. 
Linear probes recover semantic attributes from hidden states \citep{alain2016understanding, belinkov-2022-probing}, 
sparse autoencoders (SAE) identify interpretable feature directions \citep{bricken2023monosemanticity, cunningham2023sparse}, 
and single-vector activation steering reliably modifies model behavior \citep{turner2023steering, zou2023representation}. 
Across a wide range of methods and objectives, simple linear operations repeatedly succeed at isolating meaningful internal structures.

\begin{figure}[t]
    \centering
    \includegraphics[width=\linewidth]{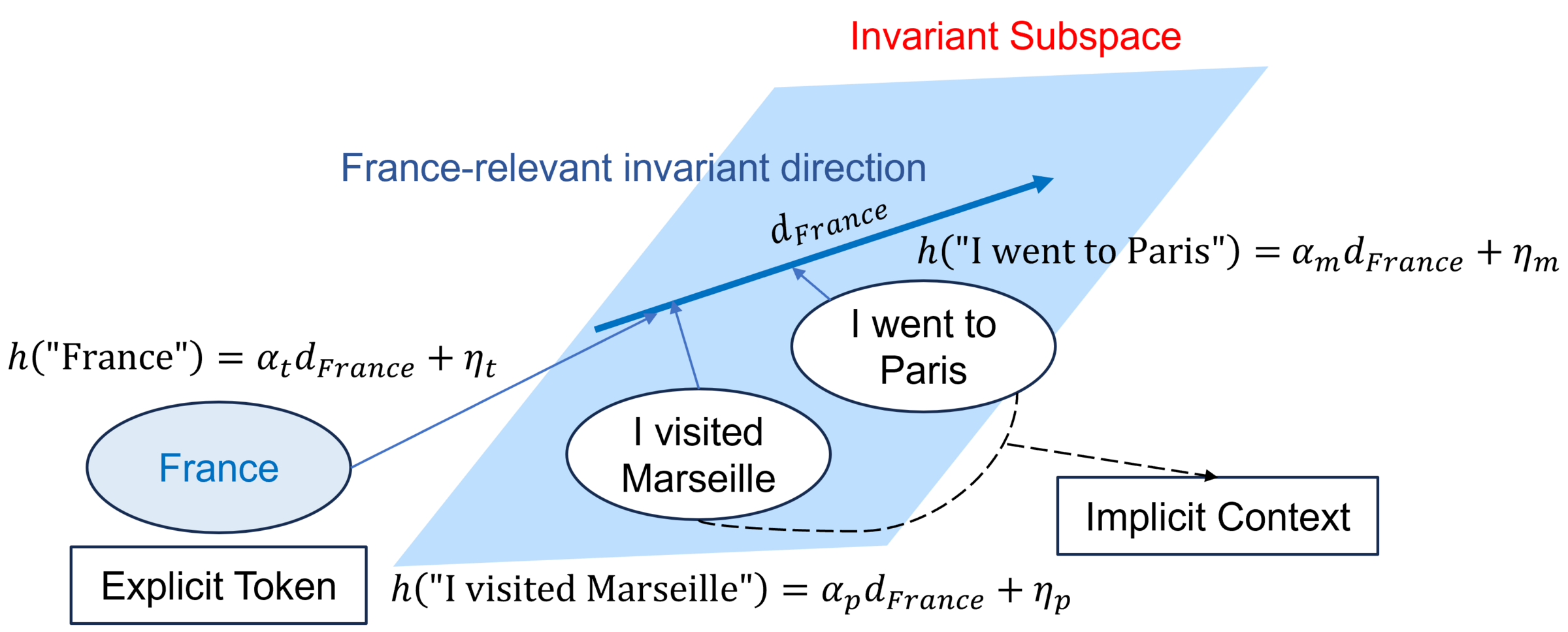}
    \caption{
    \textbf{Context-invariant directional representation.}
    The explicit token “France” provides a reference vector for this direction \textit{(self-reference)},
    while contextual mentions such as “I went to Paris” and “I visited Marseille” share the same
    invariant direction.
    \label{fig:method_overview}
    }
\end{figure}

Yet this success raises a fundamental theoretical puzzle. 
Transformers are deep, highly nonlinear systems trained with simple objectives but exhibiting complex emergent behavior. 
Why, then, should their internal representations admit such simple and reliable linear access to semantic information?
Recent theoretical work suggests that linear representations can emerge from the next-token prediction objective together with the implicit bias of gradient descent \citep{jiang2024origins}. 
We provide a complementary explanation: transformers communicate information through linear interfaces—most notably the attention OV circuit and the unembedding matrix—and any semantic feature read out through such an interface must reside in a context-invariant linear subspace. While optimization determines \emph{how} representations are learned, architecture constrains \emph{what form} they must take.

We formalize this claim in our central theoretical result, the \emph{Invariant Subspace Necessity} theorem: whenever a semantic feature is decoded through a linear interface, its representation must lie in an invariant subspace shared across all contexts expressing that feature. 
This provides a principled explanation for why linear probes, sparse autoencoders, and direction-based steering are able to recover stable semantic structure.
Crucially, if such an invariant subspace exists, how can we identify its direction in practice? 
A key derivation from it is the \emph{Self-Reference Property}: tokens and expressions that encode a feature directly provide its geometric direction in activation space. 
For example, the token ``France'' does not merely instantiate the France concept, but serves as a reference vector for locating that concept in any representation, enabling zero-shot identification of semantic directions without labeled data or learned probes.

We empirically validate these predictions across eight semantic classification tasks spanning taxonomic, affective, stylistic, linguistic, and descriptive domains, and across multiple model families (LLaMA3-8B ~\citep{grattafiori2024llama}, Mistral-7B~\citep{jiang2023mistral}, GPT2-Small~\citep{radford2019language}, LLaMA3.2-3B~\citep{meta2024llama32}).
Our results reveal consistent evidence of context-invariant directional structure, supporting the generality of our theory.

In summary, our contributions are threefold: 
(1)~We provide a principled architectural explanation for why linear 
interpretability methods succeed in transformers, showing that linear 
structure arises as a necessary consequence of linear communication 
interfaces. 
(2)~We introduce the \textit{Self-Reference Property}, 
establishing that tokens directly define directions for their associated 
features---this enables zero-shot identification of semantic directions 
and yields an unsupervised probe that classifies instances using only 
class token geometry. 
(3)~We demonstrate convergent evidence for directional invariance: 
sparse autoencoders trained without class supervision recover features 
that align with class token directions, validating that both methods 
access the same underlying structure.

% Second, we formalize this insight by proving that communicable semantic features must occupy context-invariant linear subspaces. 
\section{Related Work}

\subsection{Mechanistic Interpretability}
Mechanistic interpretability seeks to explain model behavior by identifying internal features, circuits, and causal mechanisms inside neural networks rather than relying solely on input--output correlations \citep{elhage2021mechanistic, nanda2023grokking}. 
Transformer models contain identifiable computational substructures implementing specific behaviors, and targeted interventions on internal activations can causally alter model outputs \citep{meng2022locating, wang2022causal, turner2023steering, zou2023representation}. 
Our work connects this literature to representation geometry by providing a theoretical account of the linear, context-invariant feature structure that underlies many of these observations.

\subsection{Linear Representation}
Semantic attributes can often be captured by low-dimensional linear structure in representation spaces, dating back to linear regularities in word embeddings \citep{mikolov2013linguistic, pennington2014glove}. 
In modern LMs, linear structure is operationalized as \emph{measurement} (linear probes that extract a property) and as \emph{intervention} (steering representations along a direction) \citep{kim2018tcav, ravfogel2020inlp}. 
\citet{park2024linear} formalize these intuitions, clarifying what ``linear representation'' means and how measurement and intervention relate geometrically. 
Most closely related, \citet{jiang2024origins} explain linear representations through the next-token prediction objective and implicit bias of gradient descent, whereas we give an architectural account of why linear, context-invariant directions are \emph{necessary} in models with linear communication interfaces.

\subsection{Linear Probes}
Linear probes test what information is present in hidden states \citep{alain2016understanding, hewitt2019structural, tenney2019bert}. 
A range of successful techniques instantiate this idea: the \emph{logit lens} decodes intermediate representations using the unembedding matrix, revealing that hidden states already resemble next-token distributions \citep{nostalgebraist2020}, and the tuned lens extends this with layer-specific affine transformations \citep{belrose2023tuned}. 
Relatedly, linear directions identified by probes have been used not only for measurement but also for intervention, enabling activation steering of hidden states \citep{li2023inference, rimsky2024steering}. 
Recent evaluations further show that sparse autoencoder latents do not consistently outperform simple linear baselines on probing tasks \citep{kantamneni2025sparseprobing}. 
Together, these results highlight the surprising effectiveness of linear readouts and interventions across diverse settings. 
However, they remain largely descriptive and do not explain \emph{why} such linear decodability should arise in deep nonlinear models. 
Our theory addresses this gap.

\section{Theoretical Framework: Identity-Projection}
\label{sec:theory}

We develop a geometric framework explaining how transformers encode semantic features. Our central result establishes that features communicated through linear interfaces must occupy invariant subspaces—a structural consequence of architecture—and we show that optimization dynamics favor low-dimensional realizations.

%-----------------------------------------------------------------------------
\subsection{Preliminaries and Assumptions}
%-----------------------------------------------------------------------------

We consider decoder-only transformers with the following properties:

\begin{assumption}[Architectural Requirements]
\label{ass:architecture}
Let $\mathcal{M}$ be a transformer with $L$ layers and hidden dimension $d$. We assume:
\begin{enumerate}%[label=\textbf{(A\arabic{enumi})}, leftmargin=2em]
    \item \textbf{Additive Residual Stream.} The residual stream updates additively: $\mathbf{r}^{(\ell+1)} = \mathbf{r}^{(\ell)} + \Delta^{(\ell)}$, where $\Delta^{(\ell)}$ is the output of layer $\ell$.
    
    \item \textbf{Linear Communication Interfaces.} The OV circuit in attention ($W_O W_V$) and the unembedding layer ($W_U$) act as linear maps on the residual stream.
    
    \item \textbf{Shared Parameters.} The same parameters are applied across all token positions; optimization occurs over parameters shared across contexts.
    
    \item \textbf{Linear Output Layer.} Token logits are computed via linear projection: $\text{logits} = W_U \mathbf{r}^{(L)}$.
\end{enumerate}
\end{assumption}

These assumptions hold for standard transformer architectures including GPT-2, LLaMA, and similar models.

%-----------------------------------------------------------------------------
\subsection{Formal Definitions}
%-----------------------------------------------------------------------------

\begin{definition}[Context]
\label{def:context}
A \emph{context} $c = (x_1, \ldots, x_n; i)$ consists of an input token sequence and a position of interest $i$. We write $\mathbf{h}^{(\ell)}(c) \in \mathbb{R}^d$ for the hidden state at layer $\ell$ and position $i$ in context $c$.
\end{definition}

\begin{definition}[Semantic Feature]
\label{def:feature}
A \emph{semantic feature} is a function $f: \mathcal{C} \to \mathcal{Y}$ mapping contexts to a value space $\mathcal{Y}$. For classification tasks, $\mathcal{Y} = \{y_1, \ldots, y_K\}$ is finite.
\end{definition}

\begin{definition}[Communicable Feature]
\label{def:communicable}
A semantic feature $f$ is \emph{communicable} if it satisfies both:
\begin{enumerate}%[label=(\roman*)]
    \item \textbf{Multi-context:} There exist distinct contexts $c_1, c_2$ with $f(c_1) = f(c_2)$.
    \item \textbf{Linear decodability:} The feature value is recoverable through a linear interface. For unembedding:
    \begin{equation}
        \exists\, \boldsymbol{\phi} \in \mathbb{R}^{|V|}: \quad \boldsymbol{\phi}^\top W_U \mathbf{h}(c) = g(f(c)) \quad \forall c \in \mathcal{C}
    \end{equation}
    for some function $g: \mathcal{Y} \to \mathbb{R}$.
\end{enumerate}
\end{definition}

The multi-context requirement captures that meaningful features appear across different surface realizations (e.g., ``France'' and ``the country of the Eiffel Tower'' both express the France feature). Linear decodability captures that the feature must be extractable by downstream linear operations.

\begin{definition}[Invariant Subspace]
\label{def:invariant-subspace}
A communicable feature $f$ exhibits \emph{contextual invariance} if there exists a subspace $\mathcal{S}_f \subseteq \mathbb{R}^d$, determined by $f$ and model parameters alone, such that for all contexts $c$ expressing $f$, the $f$-relevant information in $\mathbf{h}(c)$ lies entirely within $\mathcal{S}_f$.
\end{definition}

\begin{definition}[Directional Invariance]
\label{def:directional}
A communicable feature $f$ exhibits \emph{directional invariance} if it exhibits contextual invariance with $\dim(\mathcal{S}_f) = 1$, i.e., $\mathcal{S}_f = \text{span}(\{\mathbf{d}_f\})$ for some direction $\mathbf{d}_f \in \mathbb{R}^d$.
\end{definition}

\begin{figure}
    \centering
    \includegraphics[width=1\linewidth]{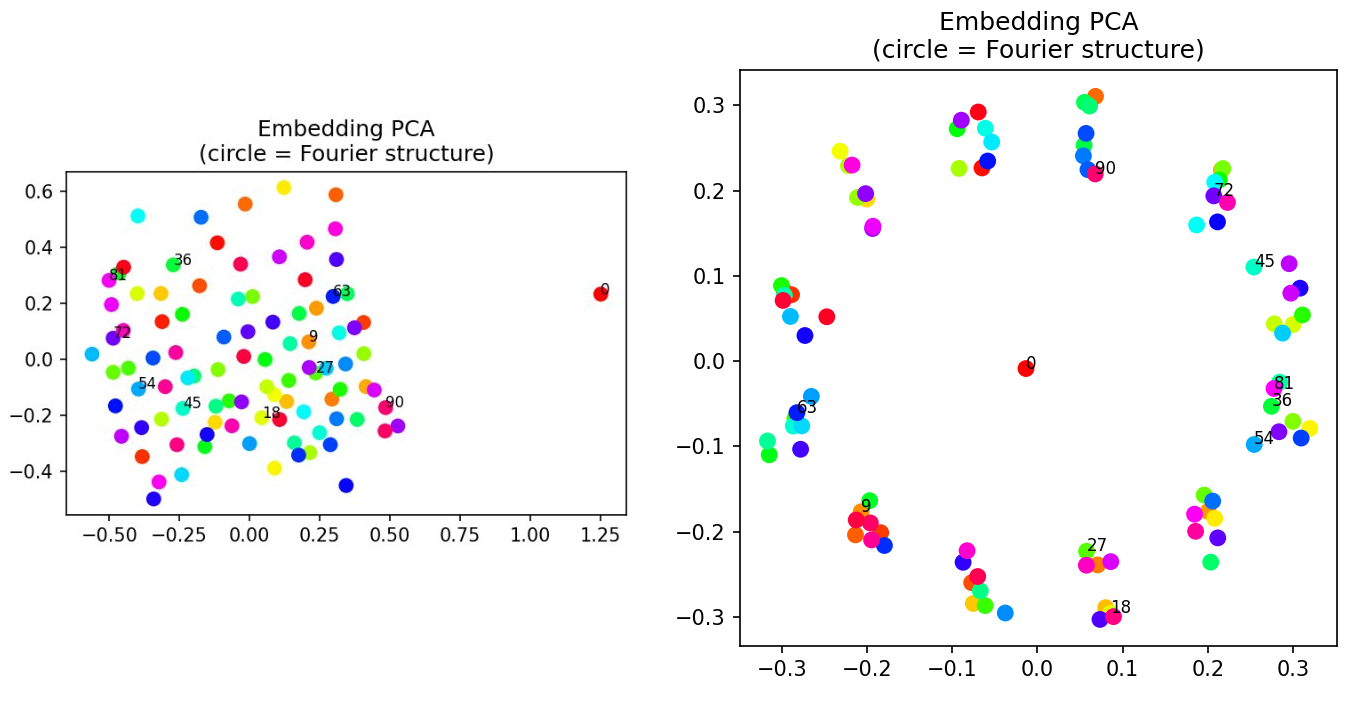}
    \caption{Linear readout layers constrain representation geometry. We train a transformer on modular division with an MLP classification head instead of linear unembedding. \textbf{(Left)} When the model finds a non-Fourier solution, embeddings lack circular structure and linear probes fail ($\sim$20\% accuracy). \textbf{(Right)} When the model discovers Fourier structure, linear probes succeed. Across random seeds, linear probe accuracy correlates with the Fourier representations emerges, but are not required by MLP heads. Linear readout interfaces would necessitate such directional structure.}

    \label{fig:fourier_analysis}
\end{figure}

%-----------------------------------------------------------------------------
\subsection{Main Theoretical Results}
%-----------------------------------------------------------------------------

Our first result establishes that invariant subspace structure is \emph{necessary} for any feature communicated through linear interfaces.

\begin{theorem}[Invariant Subspace Necessity]
\label{thm:invariant-subspace}
Let $\mathcal{M}$ be a transformer satisfying Assumption~\ref{ass:architecture}, and let $f$ be a communicable feature decoded through a linear interface $W$. Then there exists a context-invariant subspace $\mathcal{S}_f \subseteq \mathbb{R}^d$ such that the $f$-relevant component of $\mathbf{h}(c)$ lies in $\mathcal{S}_f$ for all contexts $c$ expressing $f$.
\end{theorem}

\begin{proof}[Proof sketch]
By linear decodability, the $f$-relevant output is $o_f(c) = \mathbf{w}_f^\top \mathbf{h}(c)$ for some $\mathbf{w}_f \in \mathbb{R}^d$. Contexts requiring identical outputs differ only in $\mathbf{w}_f^\perp$, so $f$-relevant information lies in a subspace determined by $\mathbf{w}_f$ alone—independent of context. Full proof in Appendix~\ref{app:proofs}.
\end{proof}

\subsection{Capacity Constraints Force Factorized Representations}

The unembedding matrix $W_U \in \mathbb{R}^{|\mathcal{V}| \times d}$ maps 
hidden states to logits over $|\mathcal{V}|$ tokens. With $|\mathcal{V}| \gg d$, 
tokens cannot occupy orthogonal directions---they must share structure.

\begin{proposition}[Capacity Constraint Implies Feature Sharing]
\label{prop:capacity}
Let $\mathcal{M}$ be a transformer with vocabulary $|\mathcal{V}|$ and 
hidden dimension $d$, where $|\mathcal{V}| \gg d$. If (i) token logits are 
computed via linear readout $\textup{logit}_t = \mathbf{w}_t^\top \mathbf{h}(c)$, 
(ii) each context activates a sparse subset of features, and (iii) multiple 
tokens share semantic attributes, then the optimal representation factorizes as:
\begin{equation}
    \mathbf{w}_t = \sum_{f \in F_t} \alpha_{t,f} \, \mathbf{d}_f
\end{equation}
where $\{\mathbf{d}_f\}$ are shared feature directions with $|F| \ll |\mathcal{V}|$.
\end{proposition}

\begin{proof}[Proof sketch]
With $|\mathcal{V}| \gg d$, tokens must share directions. Sharing incurs 
interference only when tokens co-occur; for sparse features, this cost is 
minimal. Factorization achieves $|F|$ dimensions plus sparse interference, 
versus the infeasible $|\mathcal{V}|$ dimensions for unique encodings. 
Under factorization, $\textup{logit}_t = \sum_{f \in F_t} \alpha_{t,f} 
(\mathbf{d}_f^\top \mathbf{h}(c))$, so each factor $\mathbf{d}_f$ must be 
linearly decodable and context-invariant---satisfying the conditions of 
Theorem~\ref{thm:invariant-subspace}. Full proof in Appendix~\ref{app:capacity}.
\end{proof}

\begin{remark}[Implicit Classification Revisited]
From this view, predicting a token \emph{is} implicit classification: 
the model checks ``does this context encode the factor combination for 
token $t$?'' Each factor $\mathbf{d}_f$ partitions contexts into those 
expressing $f$ versus those that do not. The compression pressure ensures 
factors are reusable across tokens, maximizing the number of features 
that can coexist in the representation.
\end{remark}

This compression pressure aligns with the information bottleneck principle~\citep{tishby2015deep}: representations must discard context-specific details while preserving task-relevant features, favoring low-dimensional factorized encodings.

\begin{corollary}[Directional Decomposition]
\label{cor:decomposition}
If feature $f$ exhibits directional invariance with direction $\mathbf{d}_f$, then for all contexts $c$ expressing $f$:
\begin{equation}
    \mathbf{h}(c) = \alpha_f(c) \cdot \mathbf{d}_f + \boldsymbol{\eta}(c)
    \label{eq:decomposition}
\end{equation}
where $\alpha_f(c) \in \mathbb{R}$ is a context-dependent magnitude and $\boldsymbol{\eta}(c) \perp \mathbf{d}_f$ captures the rest of the features.
\end{corollary}

This decomposition is central: the \emph{direction} $\mathbf{d}_f$ is invariant across contexts, while the \emph{magnitude} $\alpha_f(c)$ varies. Linear interfaces preserve this structure:
\begin{equation}
    W(\alpha \cdot \mathbf{d}_f) = \alpha \cdot (W\mathbf{d}_f)
\end{equation}

\begin{figure*}
    \centering
    \includegraphics[width=0.24\linewidth]{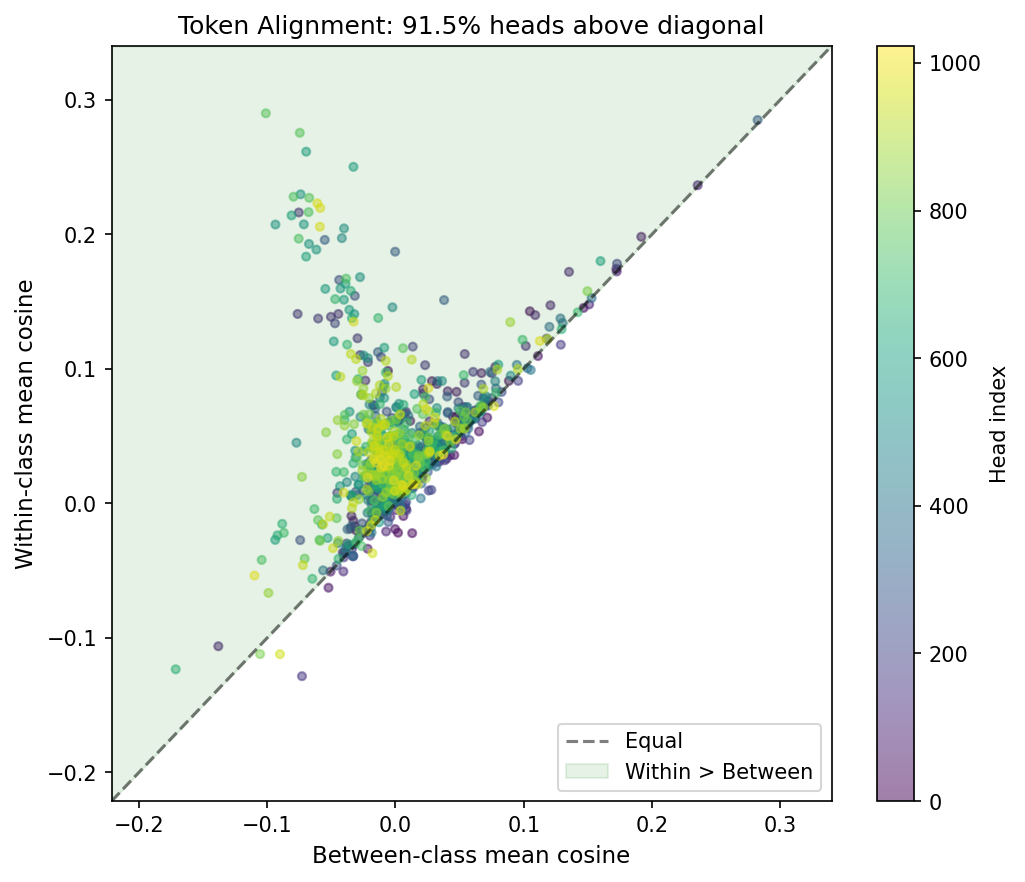}
    \includegraphics[width=0.24\linewidth]{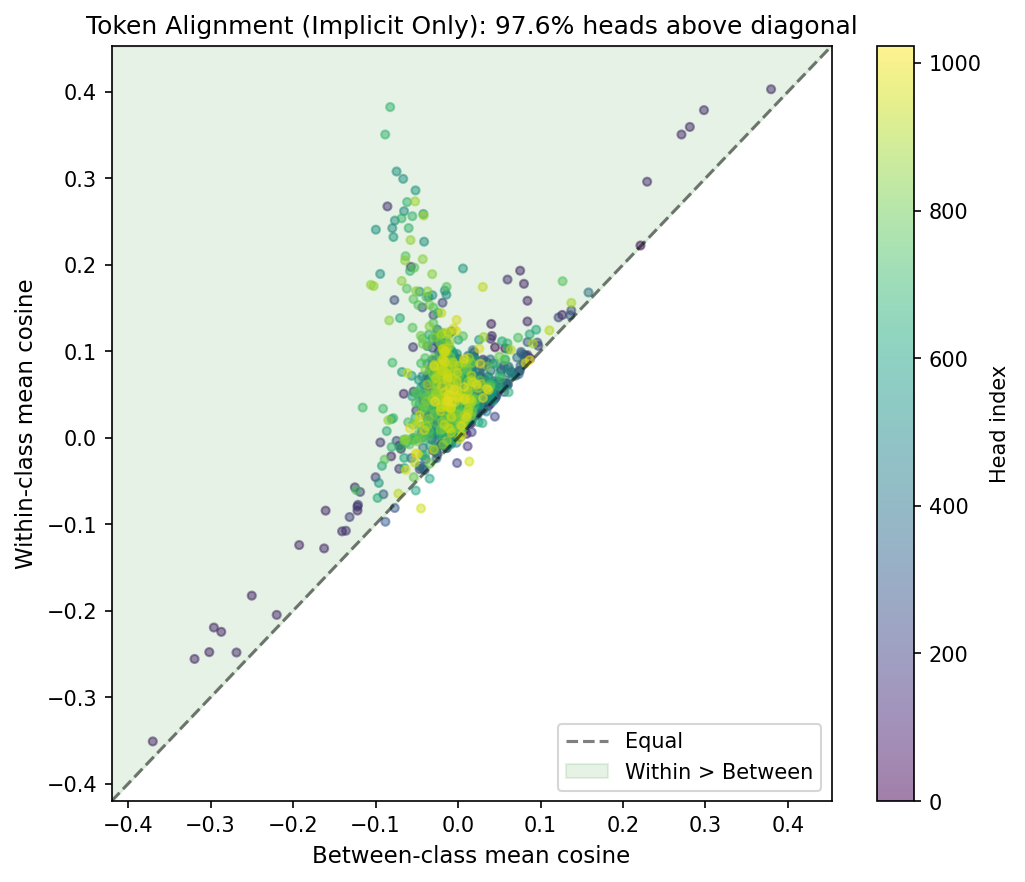}
    \includegraphics[width=0.24\linewidth]{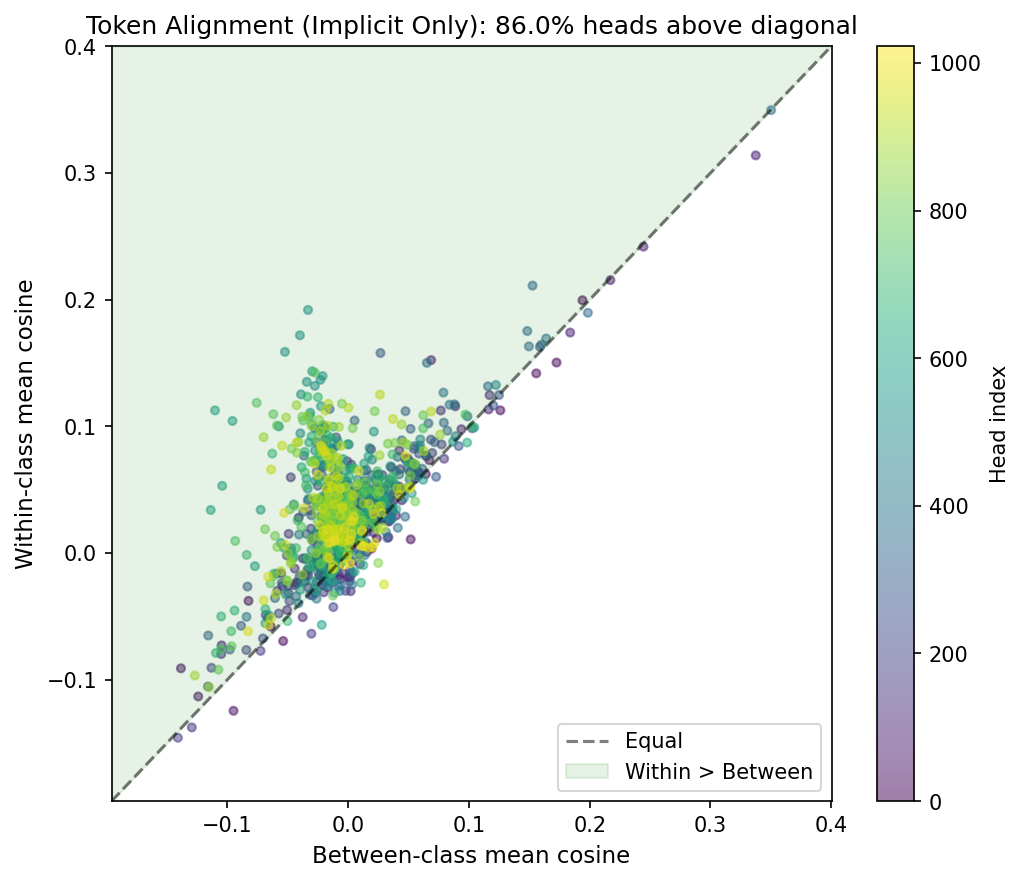}
    \includegraphics[width=0.24\linewidth]{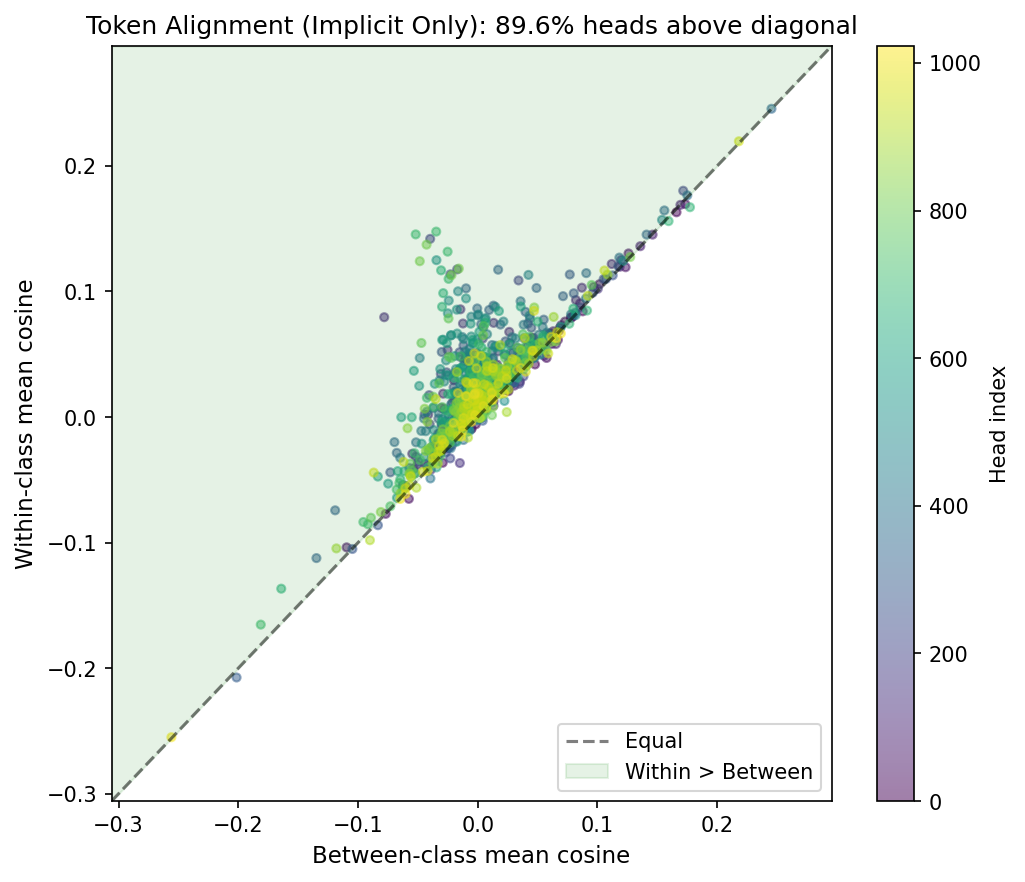}
    \caption{Token alignment validation of the Self-Reference Property across four 
    datasets in LLaMA3-8B. Each point represents one attention head; the x-axis shows 
    mean cosine similarity between class tokens and \emph{other}-class implicit instances, 
    while the y-axis shows similarity to \emph{same}-class implicit instances. Points 
    above the diagonal indicate stronger alignment with the correct class. Percentages 
    indicate heads above diagonal: \textbf{Countries} 91.5\%, \textbf{Animals} 97.6\%, 
    \textbf{Cartoon Characters} 86.0\%, \textbf{Emotions} 89.6\%.}
    \label{fig:cross_within_token}
\end{figure*}

%-----------------------------------------------------------------------------
\subsection{The Identity-Projection Operator}
%-----------------------------------------------------------------------------

We now introduce the operator that enables practical applications.

\begin{definition}[Identity-Projection Operator]
\label{def:identity-projection}
For a context $c$ and feature $f$ with invariant direction $\mathbf{d}_f$, the \emph{Identity-Projection} is:
\begin{equation}
\label{eq:identity_operator}
    \mathcal{I}_f(c) \triangleq \mathbf{h}(c)^\top \hat{\mathbf{d}}_f
\end{equation}
where $\hat{\mathbf{d}}_f = \mathbf{d}_f / \|\mathbf{d}_f\|$ is the unit direction.
\end{definition}

The Identity-Projection operator \emph{focuses} on feature $f$ within the representation $\mathbf{h}(c)$. By architectural necessity (Theorem~\ref{thm:invariant-subspace}), $f$ flows through $\mathbf{h}(c)$ along direction $\mathbf{d}_f$; projecting onto this direction extracts the feature's signal from the superposition of all features present.

\begin{proposition}[Feature Operations]
\label{prop:operations}
If feature $f$ satisfies directional invariance, then:
\begin{enumerate}[label=(\roman*)]
    \item \textbf{Detection:} $\mathcal{I}_f(c) > \tau$ indicates $f \in c$ for an appropriate threshold $\tau$
    \item \textbf{Measurement:} $|\mathcal{I}_f(c)|$ quantifies feature strength
    % \item \textbf{Removal:} $\mathbf{h}(c) - \mathcal{I}_f(c) \cdot \hat{\mathbf{d}}_f$ attenuates feature $f$
    % \item \textbf{Injection:} $\mathbf{h}(c) + \lambda \cdot \mathbf{d}_f$ amplifies feature $f$
\end{enumerate}
\end{proposition}

\begin{proof}
From the decomposition \eqref{eq:decomposition}:
\begin{equation}
    \mathcal{I}_f(c) = \alpha_f(c)\|\mathbf{d}_f\| + \boldsymbol{\eta}(c)^\top \hat{\mathbf{d}}_f
\end{equation}
When features are approximately orthogonal ($\boldsymbol{\eta}(c)^\top \hat{\mathbf{d}}_f \approx 0$), we have $\mathcal{I}_f(c) \propto \alpha_f(c)$, and the operations follow from linearity.
\end{proof}

\subsection{Linear Readouts Constrain Representation Geometry}
\label{sec:emp_invariance}

To validate Theorem~\ref{thm:invariant-subspace}, we compare representation geometry under linear versus nonlinear readout layers. If linear interfaces necessitate invariant subspace structure, then replacing the linear head with an MLP should relax this constraint.

\textbf{Setup.} Prior work on grokking~\citep{power2022grokking} shows that transformers learning modular arithmetic develop Fourier representations---embeddings arranged in a circle where addition becomes rotation. \citet{jiang2024origins} argue this linear structure emerges from training dynamics (gradient descent + cross-entropy) regardless of architecture. We test an alternative hypothesis: linear structure emerges because linear readout layers require it.

\textbf{Experiment.} We train a 2-layer transformer on modular division ($p=97$) with an MLP classification head instead of the standard linear unembedding. A separate linear probe is trained on the same hidden states (with gradients detached). The MLP head achieves $\sim$95\% validation accuracy, but the linear probe fails ($\sim$20\%). Fourier analysis (Fig.~\ref{fig:fourier_analysis}) confirms the embeddings lack circular structure---PCA shows scattered points rather than an ordered circle.

\textbf{Key finding.} Across random seeds, some runs do develop (see Fig.\ref{fig:fourier_analysis}) Fourier structure---and in exactly those runs, the linear probe also succeeds. This correlation demonstrates that Fourier representations are not inevitable: they emerge when the model happens to find that solution, but nonlinear readouts permit alternatives. Linear interfaces constrain representations to be directional; nonlinear interfaces permit but do not require such structure.
%-----------------------------------------------------------------------------
\subsection{The Self-Reference Property}
%-----------------------------------------------------------------------------

A key question remains: how do we obtain $\mathbf{d}_f$ without supervision? Our central insight is that tokens themselves provide these directions.

\begin{theorem}[Self-Reference Property]
\label{thm:self-reference}
Let $t$ be a token with associated semantic feature $f_t$. If $f_t$ is communicated through linear interfaces, then $t$'s representation provides the invariant direction for $f_t$:
\[
\mathbf{h}_t \propto \mathbf{d}_{f_t}
\]
where $\mathbf{h}_t$ is obtained by passing token $t$ through the model.
\end{theorem}

\begin{proof}[Proof sketch]
By Theorem~\ref{thm:invariant-subspace}, $f_t$ occupies an invariant subspace $\mathcal{S}_{f_t}$. Under directional invariance (Proposition~\ref{prop:dimensional-bound}), $\mathcal{S}_{f_t} = \text{span}(\{\mathbf{d}_{f_t}\})$. Since $t$ canonically expresses $f_t$, its representation must lie in $\mathcal{S}_{f_t}$, hence $\mathbf{h}_t = \lambda_t \mathbf{d}_{f_t}$ for some scalar $\lambda_t$. Full proof in Appendix~\ref{app:proofs}.
\end{proof}

\begin{remark}[Tokens Provide Directions]
The Self-Reference Property has a simple practical implication: tokens tell you \emph{where to look}. To find the ``France'' feature in any representation, obtain the direction from the token ``France'' and project onto it. The feature is there---flowing through invariant directions by architectural necessity---and focusing on the right direction reveals it.
\end{remark}

\begin{remark}[Why ``Identity-Projection'']
The framework's name reflects this self-referential structure: projecting onto the direction provided by token $t$ focuses on $t$'s identity---the features that define $t$. The token serves as its own reference point. 
\end{remark}

\subsection{Empirical Validation of the Self-Reference Property}

Theorem~\ref{thm:self-reference} predicts that token representations 
align with the invariant directions of their associated features. 
We test this by comparing each class token's hidden state against 
the mean instance vector for each class, computed across attention 
heads.

Specifically, for each attention head and class $k$, we compute the 
cosine similarity between the class token direction $\mathbf{h}_{t_k}$ 
and the mean instance direction $\bar{\mathbf{h}}_k$ for the same 
class (within-class) versus other classes (between-class). If 
tokens provide feature directions as predicted, within-class 
similarity should consistently exceed between-class similarity.

Figure~\ref{fig:cross_within_token} confirms this prediction: across 
89.6\%-98.6\% of attention heads, the class token aligns more strongly 
with its own class instances than with any other class. This pattern validates the Self-Reference Property---tokens 
encode the same directional features present in contexts expressing 
those features, enabling zero-shot extraction of semantic directions.

%-----------------------------------------------------------------------------

%=============================================================================
% END OF THEORY SECTION
%=============================================================================
%=============================================================================
% END OF THEORY SECTION
%=============================================================================
\section{Methods for Extracting Directions}
\label{sec:methods}

The Self-Reference Property (Theorem~\ref{thm:self-reference}) has a practical consequence: since tokens provide feature directions and communicable features are linear (Theorem~\ref{thm:invariant-subspace}), a classifier built on token directions should generalize to instances it was never trained on.

A probe anchored to the ``France'' or ``South Korea'' direction (which we call a \textbf{class token} or \textbf{class prompt}) should correctly classify implicit prompts like ``the country of the Eiffel Tower'' or ``the country of the Han River'' (which we call an \textbf{instance prompt}) ---not because it memorized this mapping, but because the France feature flows through the same invariant direction regardless of the surface form.

We compare three approaches that leverage this insight, each with different tradeoffs between simplicity and expressiveness. All three methods perform classification using the Identity-Projection operator (Definition~\ref{def:identity-projection}), projecting the query onto normalized class directions. They differ only in how representations are transformed before projection. All probes operate on individual attention-head outputs, and we evaluate each attention head independently.

\paragraph{Zero-Shot Probe.}
The most direct instantiation of the Self-Reference Property is to use class-token hidden states themselves as class directions, without any training.
To isolate class-specific components, we mean-center the token representations to remove features shared across all classes (e.g., ``being a country''):
\begin{equation}
    \hat{\mathbf{d}}_k =
    \frac{\mathbf{h}_{t_k} - \bar{\mathbf{h}}}
    {\|\mathbf{h}_{t_k} - \bar{\mathbf{h}}\|}, 
    \quad 
    \bar{\mathbf{h}} = \frac{1}{K}\sum_{j=1}^K \mathbf{h}_{t_j}.
\end{equation}
Classification is then performed by projecting instance representations onto these normalized directions.
This procedure requires no training data.

\paragraph{Unsupervised Probe.}
We train a lightweight linear probe that aligns instance activations with class-token directions using only class tokens.
Given an instance activation $\mathbf{h}$ and class-token prototypes $\{\mathbf{c}_k\}$ from the same head, the probe learns a transformation $W$ by minimizing the contrastive objective
\begin{equation}
\mathcal{L} = -\sum_{k} \log 
\frac{\exp((W\mathbf{h}_{t_k})^\top \mathbf{c}_k / \tau)}
{\sum_j \exp((W\mathbf{h}_{t_k})^\top \mathbf{c}_j / \tau)} .
\end{equation}
Crucially, $W$ is trained using \emph{only} class tokens—no instance 
labels are used. The Self-Reference Property predicts that this 
transformation should generalize: if tokens and instances share 
invariant directions, orthogonalizing tokens automatically orthogonalizes 
instances. Table~\ref{tab:probe_results} confirms this—the unsupervised probe 
matches or exceeds zero-shot performance despite never observing instances 
during training.

\paragraph{Sparse Autoencoder Probe.}
We train sparse autoencoders (SAEs) on attention-head outputs to learn reusable latent features in an unsupervised manner~\citep{cunningham2023sparse, bricken2023monosemanticity, kissane2024attentionSAE}. 
Given an attention output $\mathbf{h}\in\mathbb{R}^{d}$, the encoder produces sparse latents
$\mathbf{z}=\mathrm{TopK}(W_{\mathrm{enc}}\mathbf{h}+\mathbf{b}_{\mathrm{enc}})$
and reconstructs $\hat{\mathbf{h}}=\mathbf{z}W_{\mathrm{dec}}+\mathbf{b}_{\mathrm{dec}}$.
SAEs are trained only on implicit instance activations (no class tokens). 
Consistent with our theory, many SAE features activate for both class tokens and semantically related instances, indicating that SAEs recover the same invariant directions as token-based probes.

\paragraph{Unified Classification.}
For all methods, let $\phi(\cdot)$ denote the transformation (mean-centering, $W$, or SAE encoder). Classification uses the Identity-Projection operator with respect to the chosen class vectors:
\begin{equation}
    \hat{k} = \argmax_{k \in [K]} \, \phi(\mathbf{h}(c))^\top \hat{\phi}_k, \quad \text{where } \hat{\phi}_k = \frac{\phi(\mathbf{h}_{t_k})}{\|\phi(\mathbf{h}_{t_k})\|}
\end{equation}
The query is not normalized, so projection magnitude reflects feature strength. Strong classification across all methods would confirm that the invariant structure from Theorem~\ref{thm:invariant-subspace} is recoverable through multiple independent approaches---providing converging evidence for directional invariance and the Self-Reference Property.

\begin{table}
% \centering
\begin{tabular}{l|c|cc}
\hline
\multirow{1}{*}{\textbf{Dataset}} & \multirow{1}{*}{\textbf{Classes}}& \multirow{1}{*}{\textbf{Instances/Class}} &  \\
\hline
Animals & 6 & 50   \\
Countries & 5 & 39 \\
Emotional Sentences & 6 & 60  \\
Literary Quotes & 6 & 50  \\
Cartoon Phrases & 6 & 50 \\
Languages & 6 & 50 \\
Fruits & 4 & 50 \\
Companies & 4 & 50 \\
\hline
\end{tabular}
\caption{Summary of task-specific prompt datasets used for probing invariant
feature directions across semantic, stylistic, linguistic, and affective
domains.}
\label{tab:datasets_probes}
\end{table}
 %" spanish" "spanish"

\section{Experiments}

\begin{table*}[t]
\centering
\begin{tabular}{c c | c c c c c c c c}
\hline
 &  & \multicolumn{8}{c}{Accuracy (\%)} \\
\cline{3-10}
Model & Method & Animals & Countries & C. Chars & Authors & Langs & Emotions & Fruits & Companies\\
\hline
\multirow{4}{*}{\textbf{LLaMA3-8B}}
 & SAE          & 92.05 & 75.60 & 54.28 & 70.29 & 93.92 & 38.32 & 64.14 & 77.32 \\
 & Unsupervised & 94.32 & 79.60 & 61.29 & 75.13 & 87.80 & 53.52  & 73.64 & 81.32 \\
 & Zero-Shot    & 84.22 & 82.97 & 54.46 & 62.95 & 89.12 & 52.43 & 59.36 & 80.26 \\
 & Text Output       & 99.67 & 79.48 & 42.00 & 68.75 & 94.00 & 68.40 & 49.00 & 88.50  \\
\hline
\multirow{4}{*}{\textbf{Llama3.2-3B}}
 & SAE          & 78.81 & 82.43 & 53.09 & 49.29 & 92.63 & 35.76 & 61.26 & 72.63 \\
 & Unsupervised & 90.04 & 78.47 & 45.69 & 59.05 & 90.98 & 54.59 & 63.46 & 76.26   \\
 & Zero-Shot    & 72.24 & 80.36 & 43.96 & 53.68 & 91.50 & 46.73 & 60.63 & 74.50   \\
 & Text Output       & 97.67 & 46.15 & 33.33 & 39.37 & 62.50 & 62.20 & 51.00 & 88.50 \\
\hline
\multirow{4}{*}{\textbf{Mistral}}
 & SAE          & 82.10 & 74.48 & 49.95 & 52.69 & 97.62 & 43.64 & 56.52 & 75.72 \\
 & Unsupervised & 85.98 & 88.55 & 52.33 & 54.51 & 97.86 & 51.34 & 62.04 & 84.51    \\
 & Zero-Shot    & 75.71 & 81.37 & 47.16 & 46.23 & 98.15 & 50.15 & 65.18 & 79.03 \\
 & Text Output       & 92.33 & 33.85 & 33.33 & 40.00 & 98.67 & 75.20 & 84.50 & 89.00\\
\hline
\multirow{4}{*}{\textbf{GPT2-Small}}
 & SAE          & 26.65 & 39.62 & 20.21 & 15.73 & 84.02 & 17.69 & 30.83 & 52.59 \\
 & Unsupervised & 26.29 & 56.67 & 28.14 & 21.76 & 83.37 & 21.32 & 37.81 & 63.33    \\
 & Zero-Shot    & 25.39 & 37.62 & 20.83 & 20.11 & 88.11 & 17.37 & 33.90 & 29.06   \\
 & Text Output       & 35.33 & 25.64 & 17.00 & 12.50 & 15.67 & 49.80 & 37.50 & 29.50   \\
\hline
\end{tabular}
\caption{Classification accuracy across methods and models, showing the accuracy of the single best-performing (top-1) attention head for each method. Supervised probes were omitted due to their tendency to overfit. Text output denotes a zero-shot conditional log-likelihood baseline that scores each class by the sum of token log-likelihoods given the prompt.}
\end{table*}
\label{tab:probe_results}

\subsection{Datasets}
\label{sec:datasets}

We evaluate mainly across all the attention head outputs, on eight classification tasks spanning diverse semantic domains (Table~\ref{tab:datasets_probes}): taxonomic (Animals), geographic (Countries), affective (Emotional Sentences \citep{ghazi2015detecting}), stylistic (Literary Quotes, Cartoon Phrases), and linguistic (Languages \citep{artetxe2020cross}) and descriptive (Fruits, Companies).

For each task, we obtained a list of \emph{implicit} sentences that express the class without mentioning it (e.g., ``the country of the Eiffel Tower''). The full description of the datasets can be found in the Appendix~\ref{app:dataset}. 

\begin{tcolorbox}[
  colback=gray!15,
  colframe=gray!70,
  boxrule=0.6pt,
  arc=2pt,
  left=6pt,
  right=6pt,
  top=6pt,
  bottom=6pt
]
``The fruit associated with Newton's discovery of gravity'' $\xrightarrow{}$ \textbf{Apple}
\end{tcolorbox}

The Fruits and Companies datasets test behavior under polysemy 
(Section~\ref{sec:polysemy}). Our geometric analysis uses LLaMA3-8B, while classification results span four model families (Table~\ref{tab:probe_results}) to demonstrate generalization.

\subsection{Classification Results}
We do not aim to determine which method is best, but rather whether invariant directions alone support zero-shot and unsupervised classification.
Table~\ref{tab:probe_results} presents classification accuracy across 
methods, models, and semantic domains. Several findings emerge.

\paragraph{All methods achieve strong classification.}
Zero-shot, unsupervised, and SAE-based classification all distinguish 
between classes of similar semantic meaning, confirming that the 
invariant structure predicted by Theorem~\ref{thm:invariant-subspace} 
is recoverable through multiple independent approaches. The shared 
success across methods that differ only in how directions are 
extracted---directly from tokens, via learned orthogonalization, 
or through sparse decomposition---provides converging evidence for 
directional invariance.

\paragraph{Learned transformations improve separability.}
The unsupervised method consistently outperforms zero-shot classification, 
indicating that while token directions capture the correct semantic 
structure, learned orthogonalization improves class separability. 
This aligns with our motivation: raw token directions may share 
features (e.g., all country tokens encode ``being a country''), 
and the contrastive loss learns to disentangle these.

\begin{figure*}[t]
\centering
\begin{tabular}{cc}
    \includegraphics[width=0.48\linewidth]{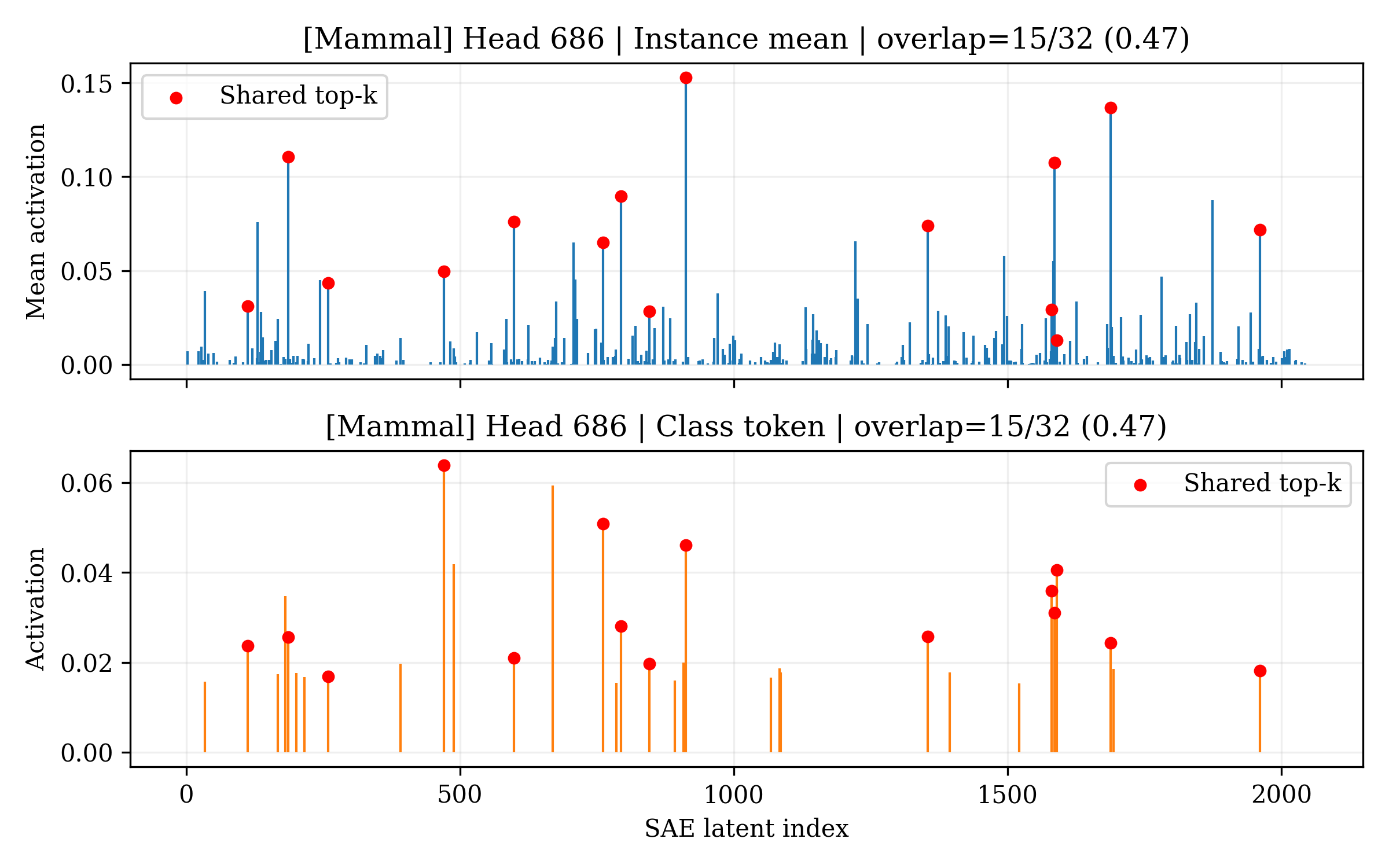} &
    \includegraphics[width=0.48\linewidth]{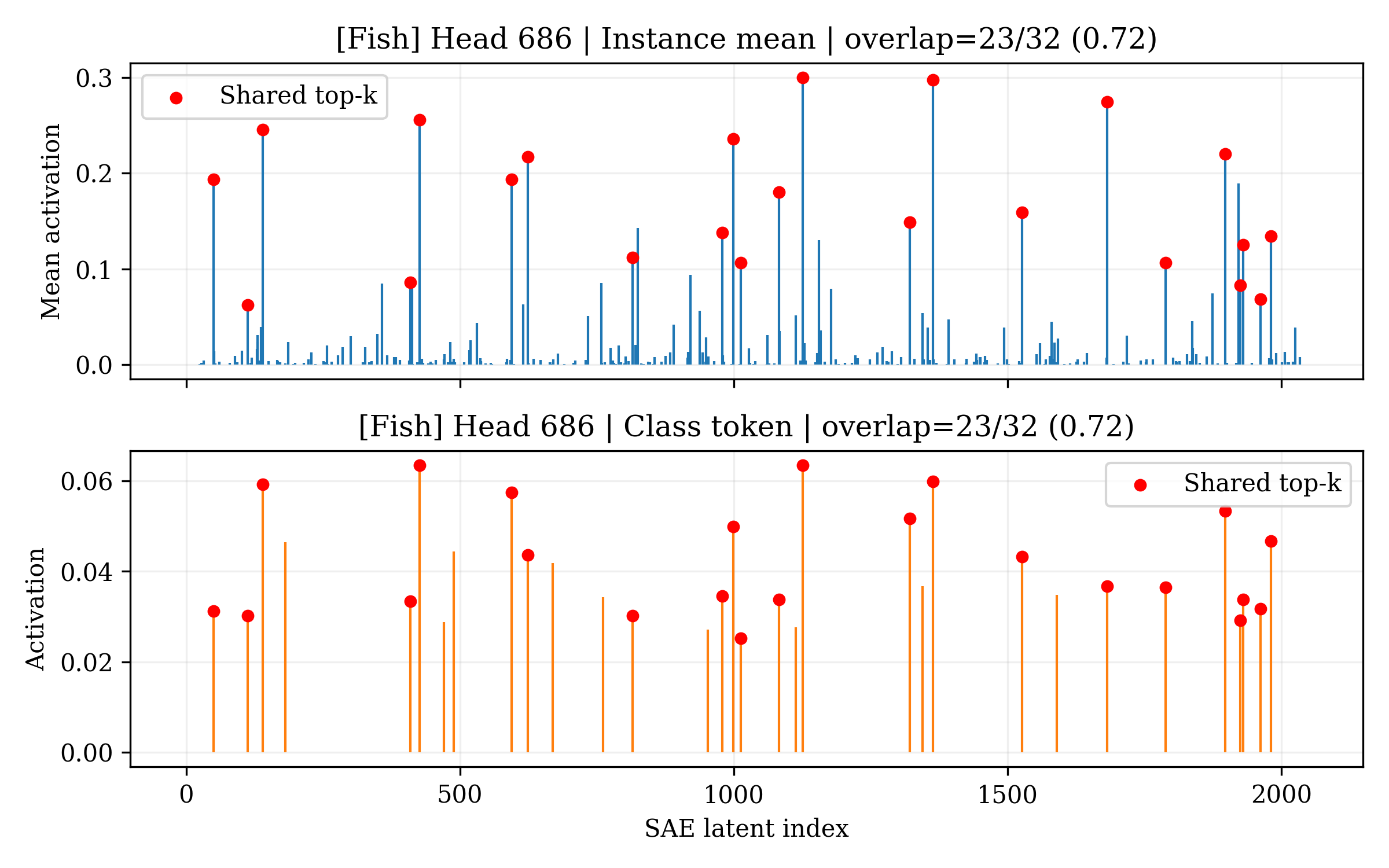}
\end{tabular}
\caption{SAE shared peak analysis across classes. We compare top-$k$ SAE dimensions of a class token with top-$k$ SAE dimensions derived from its instances. Red markers denote shared dimensions, revealing shared invariant features between tokens and contexts.}
\label{fig:overlap_sae}
\end{figure*}

\begin{figure}[t]
\centering
\includegraphics[width=0.95\columnwidth]{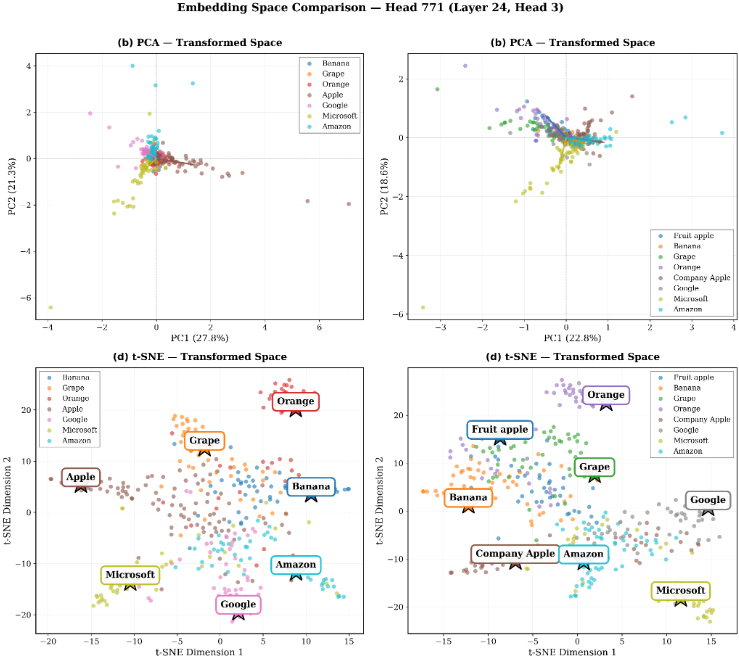}
\caption{PCA and t-SNE projections of embeddings for the polysemous word Apple (fruit vs. company) using domain-specific tokens. (Left) A single "Apple" token representing both classes (69\% accuracy). (Right) Separate tokens "Fruit apple" and "Company Apple" treated as distinct classes (65.7\% accuracy). Both methods successfully disentangle the two senses, with instances clustering around their respective class prototypes.}
\label{fig:polysemanticity}
\end{figure}

\paragraph{SAEs recover communicable features.}
SAE-based classification achieves competitive performance, 
demonstrating that features learned through reconstruction objectives 
correspond to the communicable features defined in our framework. 
This validates Theorem~\ref{thm:self-reference} from a different 
angle: SAEs trained on diverse contexts learn directions that align 
with class token directions, confirming that both methods recover 
the same underlying invariant structure. Moreover, this establishes 
a connection between our framework and SAE interpretability---the 
Identity-Projection operator can serve as a tool for analyzing 
which SAE features correspond to semantically meaningful directions.

\paragraph{Generalization across models.}
Consistent performance across LLaMA3-8B, Mistral-7B, GPT2-Small, and 
LLaMA3.2-3B suggests that directional invariance is not an artifact 
of a particular architecture or scale, but reflects a general property 
of transformer representations.

\subsection{Polysemy as the inverse of \textit{Partial Synonyms}}
\label{sec:polysemy}

Polysemy provides a natural test case: ``Apple'' encodes both fruit and company meanings as superposed directions. Table~\ref{tab:probe_results} shows both Fruits and Companies datasets classify accurately despite sharing this token. To test explicit disentanglement, we combine both datasets and compare using a single ``Apple'' token versus distinct tokens (``Fruit apple'', ``Company Apple''). Both strategies achieve comparable accuracy (Fig.\ref{fig:polysemanticity} 69\% vs.\ 65.7\%), confirming that both meanings co-exist in superposition---context modulates magnitude, not selection.

This geometric view inverts naturally: ``Dog'' and ``Cat'' are distinct tokens sharing a component along $\mathbf{d}_{\text{mammal}}$---\emph{partial synonyms} with respect to that feature. Polysemy (one token, multiple directions) and partial synonymy (multiple tokens, shared direction) are two sides of the same coin.

\subsection{Tokens Predict Feature Activation in SAEs}

We trained sparse autoencoders (SAEs) for each attention head, following the approach of \citet{kissane2024attentionSAE}, using only implicit prompts. Notably, class tokens were never introduced during training, ensuring that the SAEs learned to extract features purely from context, without direct supervision. The results are nothing short of remarkable: despite this, the SAEs still manage to retain enough structural information to reliably distinguish between classes (see Table~\ref{tab:probe_results}).

What is even more striking is how much class instances share with their corresponding class vectors. Figure~\ref{fig:overlap_sae} illustrates this phenomenon in the animal dataset, where we observe an astonishing overlap between the top-$k$ SAE dimensions of a single class token (selected by activation magnitude) and the top-$k$ most frequent SAE dimensions across its class instances. With top-$k{=}32$, the intersection between these two sets is substantial: mammals share 15/32 features, fish 23/32, reptiles 25/32, and birds 22/32. This consistent pattern emerges across multiple datasets and holds true as model size increases.

Crucially, this high overlap indicates that a class token activates the same invariant feature subspace that repeatedly emerges across diverse contextual instances of that class. In other words, tokens and their contexts converge onto a shared set of stable semantic features, suggesting that tokens themselves can serve as powerful predictors of the semantic organization learned by SAEs and transformers.

\section{Conclusion}

We have shown that linear interpretability in transformers is a structural consequence of architecture: the Invariant Subspace Necessity theorem establishes that any semantic feature communicated through linear interfaces must occupy a context-invariant subspace. This complements prior work on training dynamics \citep{jiang2024origins}—architecture constrains \emph{what} representations must look like, while optimization determines \emph{how} they are learned.

The Self-Reference Property offers a practical application: tokens directly encode their associated feature directions, enabling zero-shot semantic classification without labeled data. Conversely, instances that are not easily classifiable in this way may suggest the presence of a feature that hasn't collapsed into an invariant form. Experiments across eight domains and four model families confirm that zero-shot, a novel unsupervised probe, and SAE-based methods all recover the same directional structure.

\paragraph{Limitations.} We focus on features communicated through linear interfaces; features operating through nonlinear gating (e.g., QK routing) may not exhibit the same invariance. We also do not characterize when invariant subspaces collapse to single directions versus higher-dimensional structures.

\paragraph{Future Directions.} The Self-Reference Property suggests paths toward scalable, unsupervised circuit discovery, and the connection between SAE features and token directions offers new evaluation criteria for dictionary learning.

Interpretability methods succeed not despite transformer complexity, but because of how that complexity is structured—transformers are interpretable by design.

\section*{Impact Statement}

This paper presents work whose goal is to advance the field
of Machine Learning. There are many potential societal
consequences of our work, none which we feel must be
specifically highlighted here.

\bibliography{references}  % no .bib extension
% \bibliography{example_paper}
\bibliographystyle{icml2026}

%%%%%%%%%%%%%%%%%%%%%%%%%%%%%%%%%%%%%%%%%%%%%%%%%%%%%%%%%%%%%%%%%%%%%%%%%%%%%%%
% APPENDIX
%%%%%%%%%%%%%%%%%%%%%%%%%%%%%%%%%%%%%%%%%%%%%%%%%%%%%%%%%%%%%%%%%%%%%%%%%%%%%%%
\newpage
\appendix
\onecolumn

\section{Proofs of Main Results}
\label{app:proofs}

%-----------------------------------------------------------------------------

\subsection{Proof of Theorem~\ref{thm:invariant-subspace} (Invariant Subspace Necessity)}

\begin{proof}
Let $W \in \mathbb{R}^{m \times d}$ be the linear interface (either $W_U$ for unembedding or $W_O W_V$ for attention). By Definition~\ref{def:communicable}, there exists a linear functional $\boldsymbol{\phi} \in \mathbb{R}^m$ such that the $f$-relevant output is:
\begin{equation}
    o_f(c) = \boldsymbol{\phi}^\top W \mathbf{h}(c) = \mathbf{w}_f^\top \mathbf{h}(c)
\end{equation}
where $\mathbf{w}_f = W^\top \boldsymbol{\phi} \in \mathbb{R}^d$.

\textbf{Step 1: Equivalence class structure.} Partition contexts by their required $f$-output. For contexts $c_i, c_j$ requiring identical outputs ($o_f(c_i) = o_f(c_j)$):
\begin{equation}
    \mathbf{w}_f^\top \mathbf{h}(c_i) = \mathbf{w}_f^\top \mathbf{h}(c_j) \implies \mathbf{h}(c_i) - \mathbf{h}(c_j) \in \ker(\mathbf{w}_f^\top) = \mathbf{w}_f^\perp
\end{equation}
That is, contexts with the same feature value can differ arbitrarily in directions orthogonal to $\mathbf{w}_f$, but must agree on their projection onto $\mathbf{w}_f$.

\textbf{Step 2: Decomposition.} For any context $c$ expressing $f$, decompose:
\begin{equation}
    \mathbf{h}(c) = \mathbf{h}_f(c) + \boldsymbol{\eta}(c)
\end{equation}
where $\mathbf{h}_f(c)$ is the $f$-relevant component satisfying $\mathbf{w}_f^\top \mathbf{h}_f(c) = o_f(c)$, and $\boldsymbol{\eta}(c) \in \mathbf{w}_f^\perp$ captures $f$-irrelevant variation. This decomposition is unique given the constraint.

\textbf{Step 3: Invariant subspace construction.} The $f$-relevant components $\{\mathbf{h}_f(c_i)\}$ must satisfy $\mathbf{w}_f^\top \mathbf{h}_f(c_i) = o_f(c_i)$ for all $i$. Define:
\begin{equation}
    \mathcal{S}_f = \mathrm{span}(\{\mathbf{h}_f(c) : c \text{ expresses } f\})
\end{equation}

This subspace satisfies the required properties:
\begin{enumerate}
    \item \textbf{Context-invariance:} $\mathcal{S}_f$ is determined by feature $f$ and the linear interface $W$ (through $\mathbf{w}_f$), not by any individual context.
    \item \textbf{Completeness:} By construction, $\mathcal{S}_f$ contains all $f$-relevant information from any context expressing $f$.
    \item \textbf{Minimality:} The subspace is the span of $f$-relevant components, excluding $f$-irrelevant variation in $\mathbf{w}_f^\perp$.
\end{enumerate}

The key insight is that the linear interface $W$ acts as a bottleneck: to produce consistent outputs for the same feature value across different contexts, the model must encode $f$-relevant information in a subspace that projects identically through $\mathbf{w}_f$. This subspace exists by construction and is independent of which specific contexts are considered.
\end{proof}

\subsection{Proof of Proposition~\ref{prop:dimensional-bound} (Dimensional Bound)}

\begin{proof}
By Theorem~\ref{thm:invariant-subspace}, there exists a context-invariant subspace $\mathcal{S}_f$ containing all $f$-relevant information. We show that linear separability of $K$ feature values constrains $\dim(\mathcal{S}_f) \leq K-1$.

Let $\{\mathbf{h}_{y_k}\}_{k=1}^K \subset \mathcal{S}_f$ be representative embeddings (e.g., class centroids) for each feature value $y_k$. For these to be linearly separable, they must be affinely independent: no $\mathbf{h}_{y_k}$ lies in the affine hull of the others.

The affine hull of $K$ points has dimension at most $K-1$. To see this, define centered representations:
\begin{equation}
    \tilde{\mathbf{h}}_{y_k} = \mathbf{h}_{y_k} - \bar{\mathbf{h}}, \quad \text{where } \bar{\mathbf{h}} = \frac{1}{K}\sum_{k=1}^K \mathbf{h}_{y_k}
\end{equation}
These centered vectors satisfy $\sum_k \tilde{\mathbf{h}}_{y_k} = \mathbf{0}$, so at most $K-1$ are linearly independent. Thus $\mathrm{span}(\{\tilde{\mathbf{h}}_{y_k}\}) \leq K-1$.

Since linear separability requires the $f$-relevant subspace to distinguish all $K$ values, and this can be achieved in $K-1$ dimensions, the minimal invariant subspace satisfies $\dim(\mathcal{S}_f) \leq K-1$.

For binary features ($K=2$), this bound yields $\dim(\mathcal{S}_f) \leq 1$. Combined with the requirement that $\mathcal{S}_f$ be non-trivial (the feature must be detectable), we have $\dim(\mathcal{S}_f) = 1$, establishing directional invariance.
\end{proof}

\section{Capacity Constraints and Feature Factorization}
\label{app:capacity}

\subsection{Detailed Proof of Proposition~\ref{prop:capacity}}

We provide a complete argument that capacity constraints in transformers 
force factorized representations with shared, invariant feature directions.

\begin{proposition*}[Capacity Constraint Implies Feature Sharing, restated]
Let $\mathcal{M}$ be a transformer with vocabulary $|\mathcal{V}|$ and 
hidden dimension $d$, where $|\mathcal{V}| \gg d$. Under the following conditions:
\begin{enumerate}
    \item \textbf{Linear readout:} Token logits are computed via 
          $\textup{logit}_t = \mathbf{w}_t^\top \mathbf{h}(c)$
    \item \textbf{Sparse activation:} Each context $c$ expresses a small 
          subset of all possible features
    \item \textbf{Shared features:} Multiple tokens share semantic attributes
\end{enumerate}
Then the optimal representation factorizes tokens as:
\begin{equation}
    \mathbf{w}_t = \sum_{f \in F_t} \alpha_{t,f} \, \mathbf{d}_f
\end{equation}
where $F_t$ is the feature set for token $t$, and $\{\mathbf{d}_f\}$ are 
shared feature directions with $|F| \ll |\mathcal{V}|$.
\end{proposition*}

\begin{proof}
We proceed in four steps, establishing that factorization is optimal under 
capacity constraints when features are sparse.

\paragraph{Step 1: Dimensional constraint.}
The hidden dimension $d$ bounds the number of orthogonal directions available. 
With typical values $|\mathcal{V}| \approx 50{,}000$ and $d \approx 4{,}000$, 
at most $d$ tokens can have mutually orthogonal representations in $W_U$. 
The remaining $|\mathcal{V}| - d$ tokens must share directions with others, 
which creates potential interference in the readout.

\paragraph{Step 2: Interference cost depends on co-occurrence.}
Suppose tokens $t_1$ and $t_2$ share a direction $\mathbf{d}$, i.e., both 
$\mathbf{w}_{t_1}$ and $\mathbf{w}_{t_2}$ have non-zero projections onto 
$\mathbf{d}$. Interference occurs when both tokens are relevant predictions 
for a context $c$: the hidden state $\mathbf{h}(c)$ cannot independently 
modulate the logits for $t_1$ and $t_2$ along their shared direction. 
The expected interference cost is:
\begin{equation}
    \mathcal{L}_{\text{interference}}(t_1, t_2) \propto 
    P(t_1, t_2 \text{ both relevant in } c) \cdot 
    |\mathbf{w}_{t_1}^\top \mathbf{d}| \cdot |\mathbf{w}_{t_2}^\top \mathbf{d}|
\end{equation}
For sparse features---those with low co-occurrence probability across the 
corpus---this interference cost is small.

\paragraph{Step 3: Factorization minimizes total representation cost.}
Consider the total cost of a representation as the sum of dimensional cost 
(number of directions used) and interference cost. Compare two strategies:

\begin{itemize}
    \item \emph{Unique directions:} Each token $t$ receives a dedicated 
          direction $\mathbf{d}_t$. This eliminates interference but requires 
          $|\mathcal{V}|$ dimensions---infeasible when $|\mathcal{V}| \gg d$.
    
    \item \emph{Factorized features:} Tokens share $|F|$ feature directions, 
          where each token $t$ is represented as a sparse combination of 
          features $f \in F_t$. This requires only $|F|$ dimensions plus 
          interference costs proportional to feature co-occurrence.
\end{itemize}

When semantic features are sparse and $|F| \ll |\mathcal{V}|$, the factorized 
strategy achieves lower total cost. Intuitively, ``mammal'' and ``European'' 
rarely co-occur as the dominant semantic features of a context, so tokens 
sharing these directions (e.g., ``cat'' and ``France'') experience minimal 
interference despite their shared structure.

\paragraph{Step 4: Each factor is communicable and invariant.}
Under factorization, the logit computation for token $t$ decomposes as:
\begin{equation}
    \text{logit}_t = \mathbf{w}_t^\top \mathbf{h}(c) 
    = \sum_{f \in F_t} \alpha_{t,f} \, \mathbf{d}_f^\top \mathbf{h}(c)
    = \sum_{f \in F_t} \alpha_{t,f} \, \phi_f(c)
\end{equation}
where $\phi_f(c) = \mathbf{d}_f^\top \mathbf{h}(c)$ is the activation of 
feature $f$ in context $c$. For the model to correctly predict token 
probabilities, the hidden state $\mathbf{h}(c)$ must encode each relevant 
feature $f$ such that $\phi_f(c)$ reflects the appropriate magnitude. 
This imposes three requirements on each feature direction $\mathbf{d}_f$:

\begin{enumerate}
    \item \textbf{Linearly decodable:} The feature activation is extracted 
          via a linear projection $\mathbf{d}_f^\top \mathbf{h}(c)$.
    
    \item \textbf{Multi-context:} The same direction $\mathbf{d}_f$ is used 
          across all tokens $t$ with $f \in F_t$ and all contexts $c$ where 
          feature $f$ is relevant.
    
    \item \textbf{Context-invariant:} The direction $\mathbf{d}_f$ is fixed 
          in $W_U$ and does not depend on $c$; only the magnitude 
          $\phi_f(c)$ varies with context.
\end{enumerate}

These three properties---linear decodability, multi-context usage, and 
context-invariance---are precisely the conditions of Theorem~\ref{thm:invariant-subspace}. 
Therefore, each semantic factor $\mathbf{d}_f$ occupies an invariant subspace 
in the representation geometry.
\end{proof}

\paragraph{Connection to prior work.}
This result connects to the superposition hypothesis~\citep{elhage2022toy}, 
which demonstrates that neural networks exploit high-dimensional geometry 
to represent more features than dimensions when features are sparse. Our 
contribution is to show that this superposition, combined with the linear 
readout constraint of transformers, forces the specific geometric structure 
of invariant feature directions. While~\citet{achille2018emergence} establish 
that information-minimal representations are necessarily invariant, we 
characterize the \emph{form} this invariance takes: fixed directions with 
context-dependent magnitudes.

\subsection{Proof of Proposition~\ref{prop:layernorm-preservation}}

\begin{proposition}[LayerNorm preserves directional invariance]
\label{prop:layernorm-preservation}
Let $h(c) = \alpha_f(c)\, d_f + \eta(c)$ where $d_f \in \mathbb{R}^d$ is a 
context-invariant feature direction and $\alpha_f(c) \in \mathbb{R}$ is a 
context-dependent coefficient. If $d_f \notin \mathrm{span}(\mathbbm{1})$, 
then there exists a context-invariant direction $\tilde{d}_f$ such that
\[
\mathrm{LN}(h(c)) = \tilde{\alpha}_f(c)\, \tilde{d}_f + \tilde{\eta}(c) + \beta,
\]
where $\tilde{\alpha}_f(c) := \alpha_f(c)/\sigma(h(c))$ and $\beta$ is the 
learned bias.
\end{proposition}

\begin{proof}
We analyze each component of Layer Normalization.

\paragraph{Step 1: Mean-centering.}
Define the mean-centering projection
\[
\Pi_{\mathbbm{1}^\perp} := I - \frac{1}{d}\mathbbm{1}\mathbbm{1}^\top,
\]
where $\mathbbm{1} \in \mathbb{R}^d$ is the all-ones vector. Applying this 
linear operator to our decomposition:
\[
\Pi_{\mathbbm{1}^\perp} h(c) 
= \alpha_f(c)\, \Pi_{\mathbbm{1}^\perp} d_f + \Pi_{\mathbbm{1}^\perp} \eta(c).
\]
Since $d_f \notin \mathrm{span}(\mathbbm{1})$ by assumption, the projected 
direction $d_f' := \Pi_{\mathbbm{1}^\perp} d_f \neq 0$.

\paragraph{Step 2: Scalar normalization.}
LayerNorm divides by the standard deviation $\sigma(h(c)) > 0$:
\[
\frac{\Pi_{\mathbbm{1}^\perp} h(c)}{\sigma(h(c))}
= \frac{\alpha_f(c)}{\sigma(h(c))}\, d_f' 
  + \frac{\Pi_{\mathbbm{1}^\perp} \eta(c)}{\sigma(h(c))}.
\]
Since $\sigma(h(c))$ is a positive scalar, this operation preserves the 
subspace $\mathrm{span}(d_f')$.

\paragraph{Step 3: Learned affine transformation.}
LayerNorm applies elementwise scaling by $\gamma \in \mathbb{R}^d$ and 
adds bias $\beta \in \mathbb{R}^d$:
\[
\mathrm{LN}(h(c)) 
= \frac{\alpha_f(c)}{\sigma(h(c))}\, (\gamma \odot d_f')
  + \frac{\gamma \odot \Pi_{\mathbbm{1}^\perp} \eta(c)}{\sigma(h(c))} 
  + \beta.
\]
Define the transformed feature direction
\[
\tilde{d}_f := \gamma \odot d_f' = \gamma \odot \Pi_{\mathbbm{1}^\perp} d_f.
\]
This direction depends only on $d_f$ and the learned parameter $\gamma$, 
and is therefore context-invariant.

LayerNorm maps the context-invariant subspace $\mathrm{span}(d_f)$ to a 
new context-invariant subspace $\mathrm{span}(\tilde{d}_f)$ via the fixed 
transformation $d_f \mapsto \gamma \odot \Pi_{\mathbbm{1}^\perp} d_f$. 
The feature remains affinely decodable: a linear probe along $\tilde{d}_f$ 
with appropriate bias recovers a signal proportional to $\alpha_f(c)$, 
scaled by $1/\sigma(h(c))$.
\end{proof}

\subsection{Proof of Proposition~\ref{prop:dimensional-bound} (Dimensional Bound)}
\begin{proposition}[Dimensional Bound]
\label{prop:dimensional-bound}
If $f$ takes $K$ linearly separable values, then $\dim(\mathcal{S}_f) \leq K - 1$. For binary features, this yields directional invariance: $\mathcal{S}_f = \mathrm{span}(\{\mathbf{d}_f\})$.
\end{proposition}
\subsection{Proof of Theorem~\ref{thm:self-reference} (Self-Reference Property)}
%-----------------------------------------------------------------------------

\begin{theorem*}[Self-Reference Property]
Let $t$ be a token with associated semantic feature $f_t$. If $f_t$ is communicated through linear interfaces, then $t$'s representation provides the invariant direction for $f_t$:
\[
\mathbf{h}_t \propto \mathbf{d}_{f_t}
\]
where $\mathbf{h}_t$ is obtained by passing token $t$ through the model.
\end{theorem*}

\begin{proof}
We establish the result in three steps.

\textbf{Step 1: Existence of invariant subspace.}
By Theorem~\ref{thm:invariant-subspace}, any communicable feature $f_t$ is associated with an invariant subspace $\mathcal{S}_{f_t} \subseteq \mathbb{R}^d$ that contains all $f_t$-relevant information. This subspace is determined by $f_t$ and the model parameters, independent of any particular context.

\textbf{Step 2: Directional invariance for tokens.}
Under the assumption of directional invariance (which holds for binary features by Proposition~\ref{prop:dimensional-bound}, and empirically for many semantic features), this subspace is one-dimensional:
\[
\mathcal{S}_{f_t} = \text{span}(\{\mathbf{d}_{f_t}\})
\]
where $\mathbf{d}_{f_t} \in \mathbb{R}^d$ is the unique direction (up to sign) representing $f_t$.

\textbf{Step 3: Token as canonical expression.}
The token $t$ is a canonical expression of feature $f_t$---by definition, presenting $t$ to the model expresses $f_t$ maximally and unambiguously. Therefore, when $t$ is processed, the resulting representation $\mathbf{h}_t$ must encode $f_t$.

By the invariant subspace property, all $f_t$-relevant information lies in $\mathcal{S}_{f_t}$. Since $t$ canonically expresses $f_t$, the $f_t$-component of $\mathbf{h}_t$ must be non-zero and lie entirely within $\mathcal{S}_{f_t} = \text{span}(\{\mathbf{d}_{f_t}\})$. Thus:
\[
\mathbf{h}_t = \lambda_t \mathbf{d}_{f_t} + \boldsymbol{\eta}_t
\]
where $\lambda_t \neq 0$ is a scalar magnitude and $\boldsymbol{\eta}_t \perp \mathbf{d}_{f_t}$ captures features orthogonal to $f_t$.

For tokens that primarily express a single dominant feature (e.g., ``France'' primarily expresses the France feature), $\boldsymbol{\eta}_t$ is small relative to $\lambda_t \mathbf{d}_{f_t}$, yielding $\mathbf{h}_t \propto \mathbf{d}_{f_t}$ to good approximation.

\textbf{Consistency across contexts.}
For any context $c$ expressing $f_t$, the same invariant direction applies:
\[
\mathbf{h}(c) = \lambda_c \mathbf{d}_{f_t} + \boldsymbol{\eta}(c)
\]
where $\lambda_c$ varies with context but $\mathbf{d}_{f_t}$ remains fixed. This confirms that the token direction $\mathbf{h}_t$ and context directions $\mathbf{h}(c)$ share the same orientation, differing only in magnitude---exactly as predicted by directional invariance.
\end{proof}

%-----------------------------------------------------------------------------
\subsection{Additional Theoretical Implications}
%-----------------------------------------------------------------------------

Our framework yields several additional consequences for transformer interpretability:

\begin{corollary}[Attention Preserves Feature Identity]
\label{cor:attention}
The OV circuit linearly transforms features while preserving their identity. If position $j$ encodes feature $f$ as $\mathbf{h}_j = \alpha_f \mathbf{d}_f + \boldsymbol{\eta}_j$, the attention output at position $i$ is:
\begin{equation}
    \mathbf{o}_i^{(f)} = \underbrace{\left(\sum_j a_{ij} \cdot \alpha_f^{(j)}\right)}_{\text{context-dependent magnitude}} \cdot \underbrace{W_O W_V \mathbf{d}_f}_{\text{linearly transformed direction}}
\end{equation}
The QK circuit modulates \emph{how much} of the feature transfers (magnitude); the OV circuit determines \emph{how} the feature direction is transformed.
\end{corollary}

\begin{proof}
By linearity of the OV circuit:
\begin{align}
\mathbf{o}_i &= \sum_j a_{ij} \cdot W_O W_V \mathbf{h}_j \\
&= \sum_j a_{ij} \cdot W_O W_V (\alpha_f^{(j)} \mathbf{d}_f + \boldsymbol{\eta}_j) \\
&= \left(\sum_j a_{ij} \cdot \alpha_f^{(j)}\right) W_O W_V \mathbf{d}_f + \sum_j a_{ij} \cdot W_O W_V \boldsymbol{\eta}_j
\end{align}
The $f$-relevant component is the first term, with magnitude determined by attention-weighted sum and direction determined by linear transformation of $\mathbf{d}_f$.
\end{proof}

\begin{corollary}[Distributional Influence]
\label{cor:steering}
Adding a feature direction $\mathbf{d}_f$ to the hidden state shifts the output distribution predictably:
\begin{equation}
    \Delta\text{logit}_t = \lambda \cdot \mathbf{w}_t^\top \mathbf{d}_f
\end{equation}
Tokens aligned with $\mathbf{d}_f$ are boosted; anti-aligned tokens are suppressed. This provides the theoretical foundation for activation steering~\citep{turner2023steering, zou2023representation}.
\end{corollary}

\begin{proof}
Let $\mathbf{h}$ be the original hidden state and $\mathbf{h}' = \mathbf{h} + \lambda \mathbf{d}_f$ the steered state. By linearity of unembedding:
\[
\text{logit}_t(\mathbf{h}') = \mathbf{w}_t^\top \mathbf{h}' = \mathbf{w}_t^\top \mathbf{h} + \lambda \mathbf{w}_t^\top \mathbf{d}_f = \text{logit}_t(\mathbf{h}) + \lambda \mathbf{w}_t^\top \mathbf{d}_f
\]
Thus $\Delta\text{logit}_t = \lambda \mathbf{w}_t^\top \mathbf{d}_f$, which is positive when $\mathbf{w}_t$ and $\mathbf{d}_f$ are aligned.
\end{proof}

\begin{corollary}[Non-Bidirectionality]
\label{cor:reversal}
The invariant subspace for ``$A$ predicts $B$'' is generally distinct from that for ``$B$ predicts $A$'':
\begin{equation}
    \mathcal{S}_{A \to B} \not\subseteq \mathcal{S}_{B \to A} \quad \text{and} \quad \mathcal{S}_{B \to A} \not\subseteq \mathcal{S}_{A \to B}
\end{equation}
This provides a geometric explanation for the reversal curse~\citep{berglund2023reversal}.
\end{corollary}

\begin{proof}
The features ``$A$ predicts $B$'' and ``$B$ predicts $A$'' are distinct communicable features---they produce different outputs and are learned from different training examples. By Theorem~\ref{thm:invariant-subspace}, each occupies its own invariant subspace determined by that feature. Since these are different features, their subspaces need not coincide. Models trained on ``$A$ is $B$'' learn $\mathcal{S}_{A \to B}$ but have no reason to simultaneously learn $\mathcal{S}_{B \to A}$, explaining the failure to infer reverse relations.
\end{proof}

\section{Additional Testing on Unsupervised Probes}

\subsection{Unsupervised Probes}

We validate these properties through comprehensive visualization analysis, comparing embeddings before and after probe transformation using both PCA and t-SNE projections. These visualizations confirm that unsupervised probes learn transformations that improve class separability—instances cluster more tightly around their corresponding prototypes in the transformed space—while the prototype vectors themselves maintain interpretable directional structure aligned with semantic features.

\begin{figure}[H]
    \centering
    \includegraphics[width=1\linewidth]{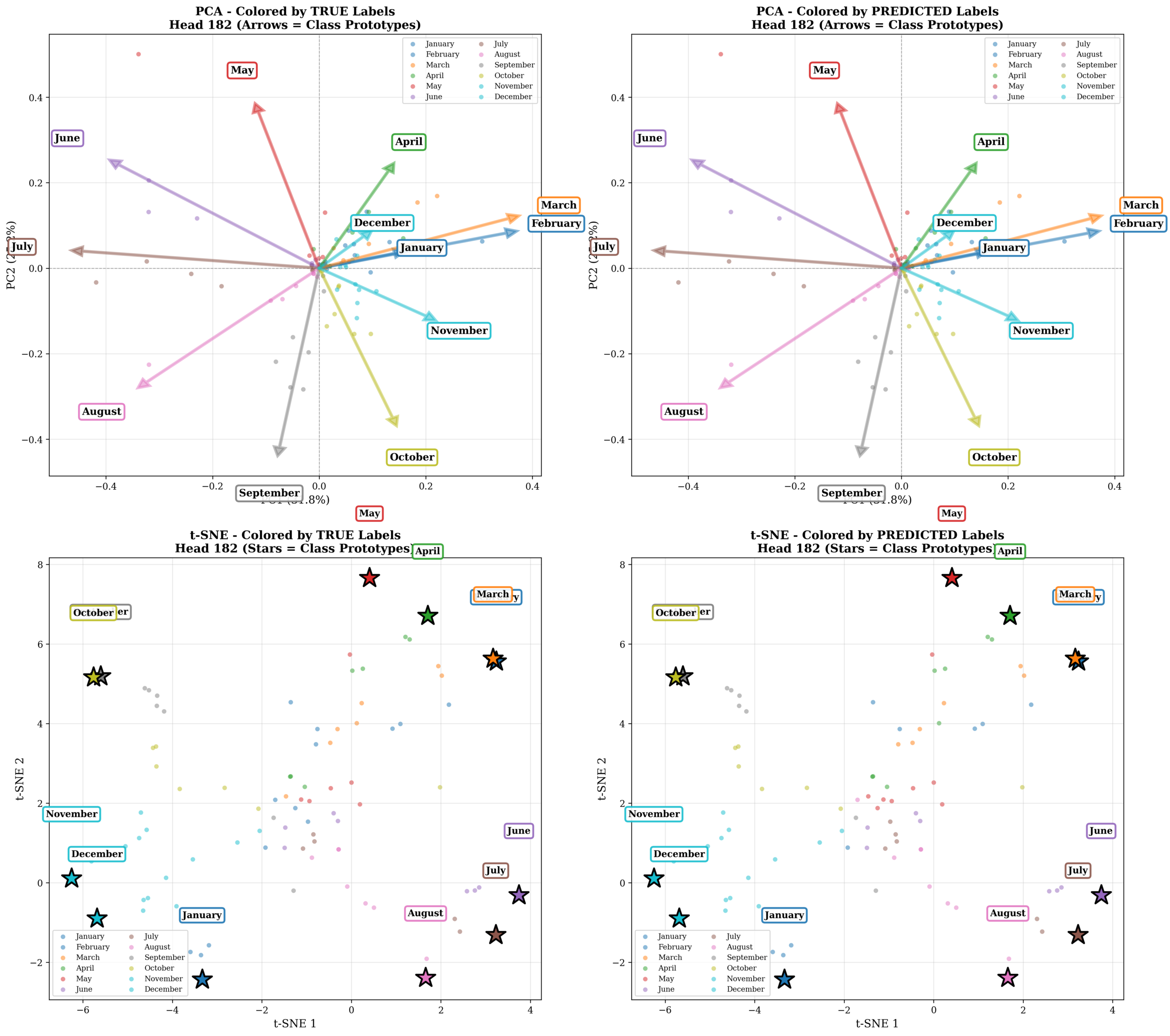}
    \caption{Zero-shot generalization of unsupervised probes. Probes trained only on odd months (Jan, Mar, May, Jul, Sep, Nov) successfully classify held-out even months, demonstrating that the learned transformation captures generalizable temporal structure rather than memorizing training classes.}
    \label{fig:months}
\end{figure}

Zero-shot generalization to unseen classes. PCA (top) and t-SNE (bottom) projections of month embeddings after unsupervised probe transformation. Training used only odd months; even months were held out. In both visualizations, held-out classes cluster tightly around their respective prototypes, demonstrating that the learned geometric structure generalizes without exposure to these classes during training.

\begin{figure*}[H]
\centering
\includegraphics[width=0.45\textwidth]{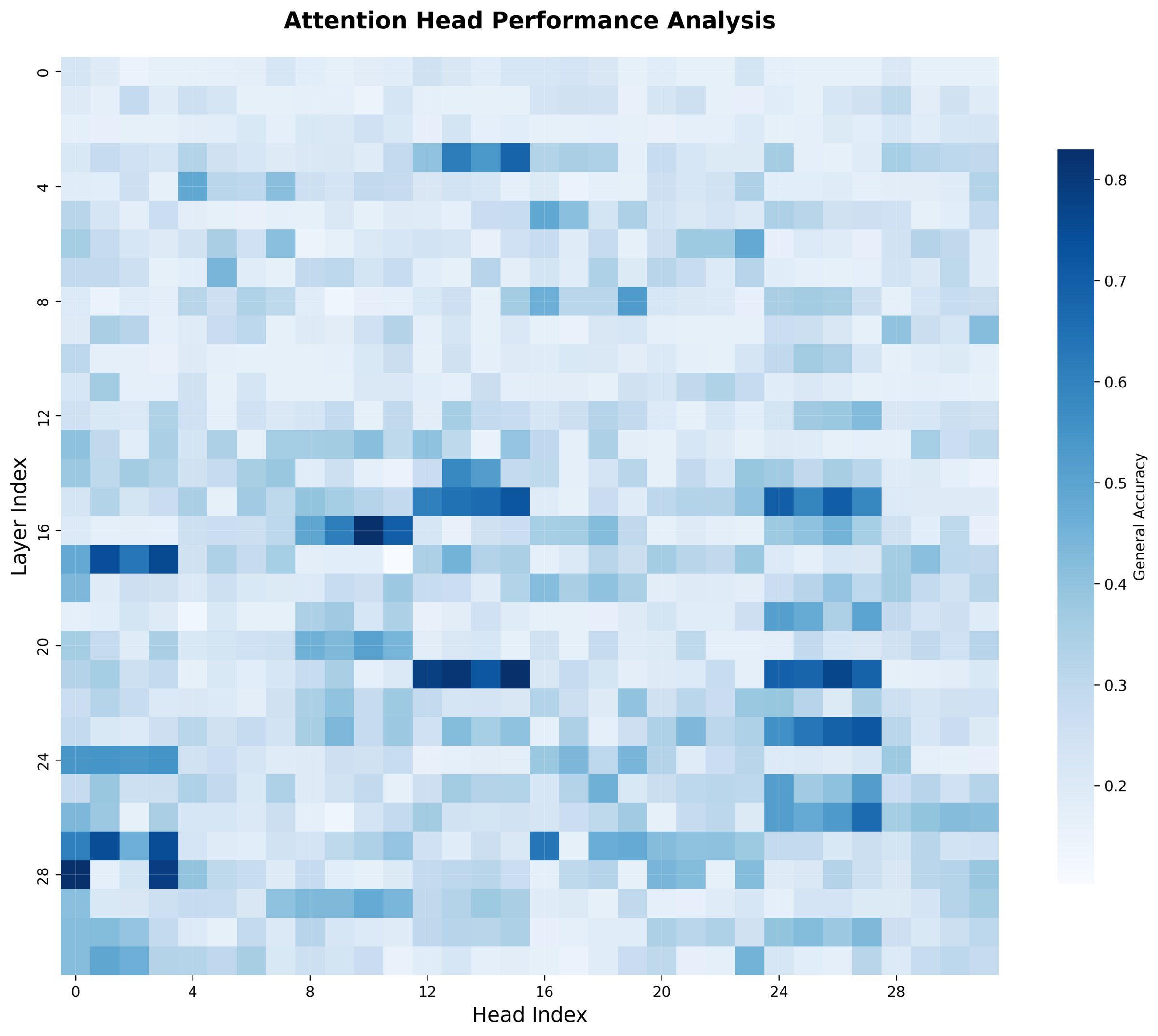}
\hfill
\includegraphics[width=0.45\textwidth]{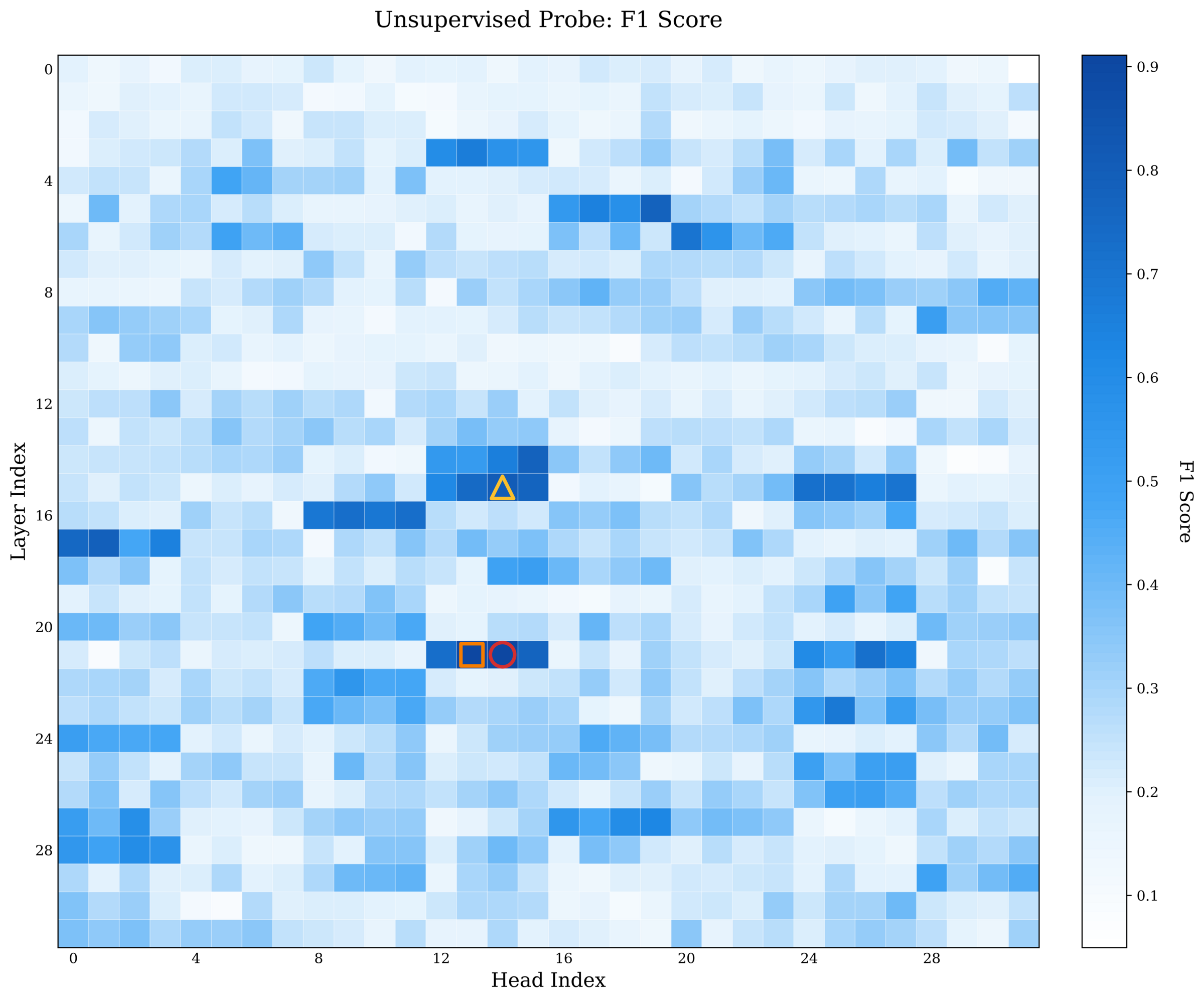}
\caption{Zero-shot (left) and Unsupervised prob classification performance on the Animal dataset in Llama-3-8B. Both contain similar accuracy distributions.}
\label{fig:heatmap_comp}
\end{figure*}

\section{Supplementary Visualizations}
\label{app:sae}
This section provides additional SAE-based analyses supporting the invariant feature sharing and head-level sparsity patterns discussed in the main paper.

Across different datasets and models, we report (i) shared top-$k$ SAE latent dimensions between instance-mean and class-prototype representations within individual heads, and (ii) unsupervised head-level classification performance using SAE latents. Figures are grouped by dataset and model to demonstrate the consistency of these patterns.

% SAE -> Llama8b -> Fruits
\begin{figure}[H]
\centering
\begin{tabular}{cc}
    \includegraphics[width=0.52\linewidth]{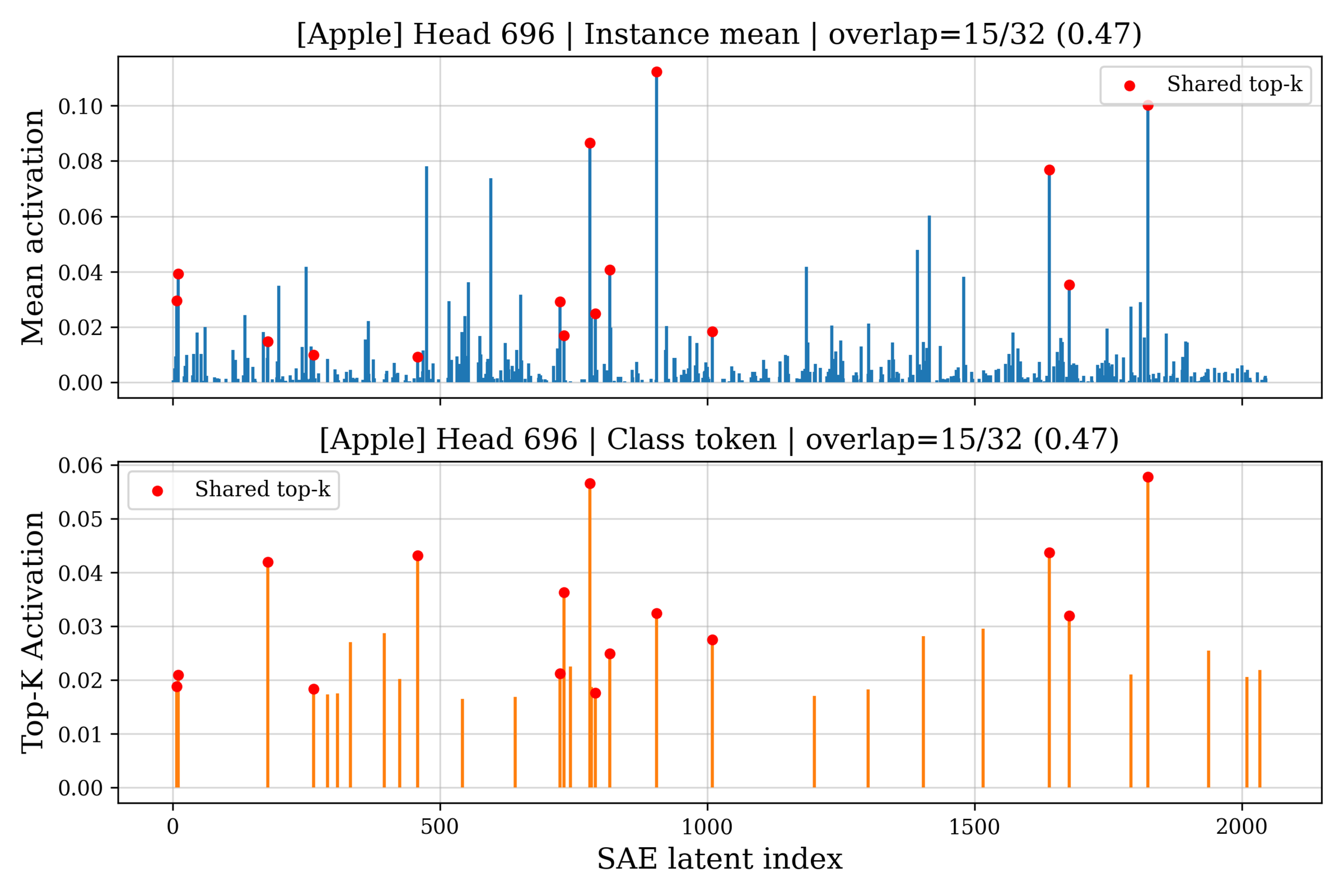} &
    \includegraphics[width=0.52\linewidth]{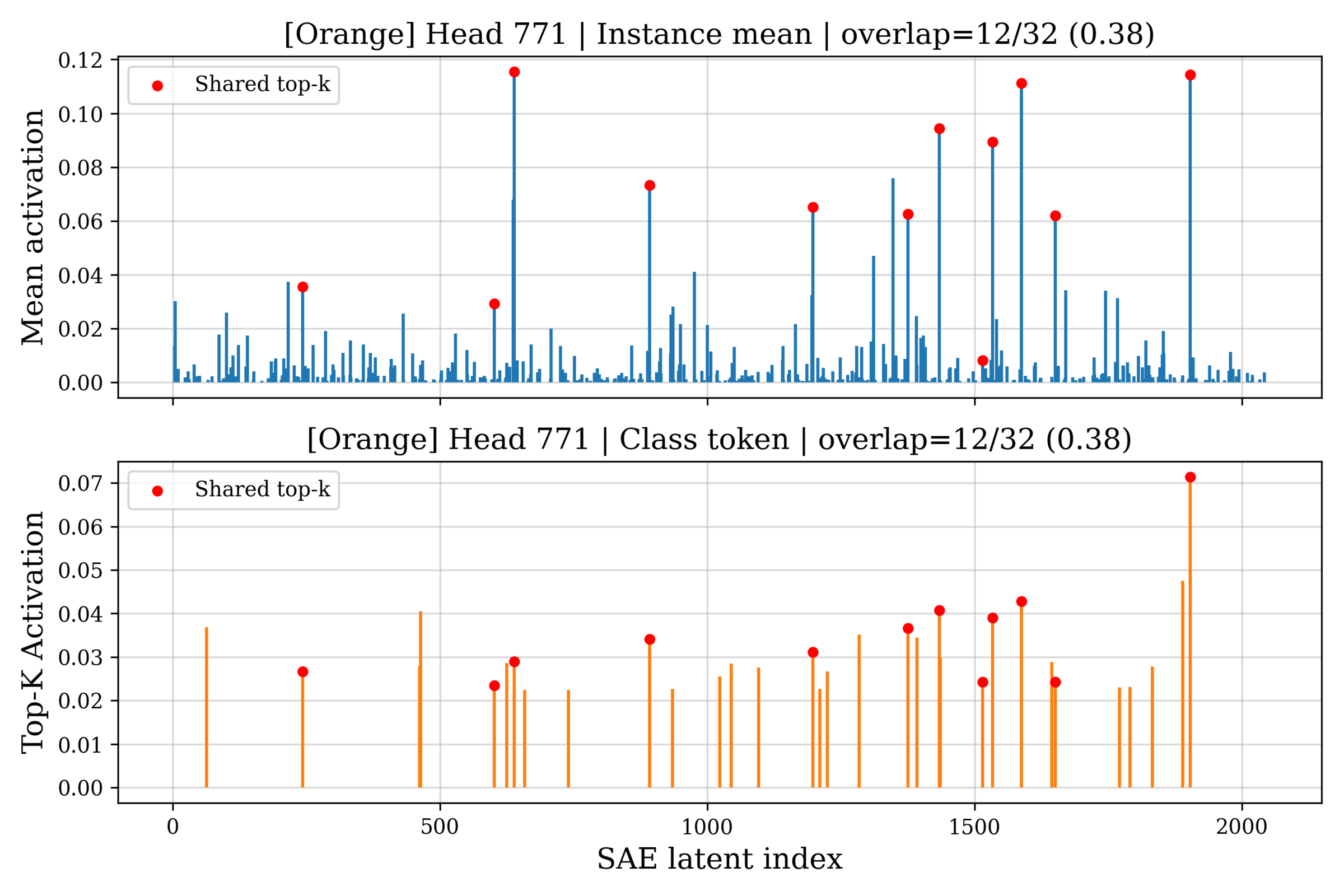}
\end{tabular}
\caption{Shared top-$k$ SAE latent activations between instance-mean and class-prototype representations for the Apple and Orange classes in Llama-3-8B, highlighting class-consistent sparse features within individual attention heads.}
\label{fig:sae_llama8b_fruits_overlap}
\end{figure}

\begin{figure}[H]
\centering
\includegraphics[width=0.7\columnwidth]{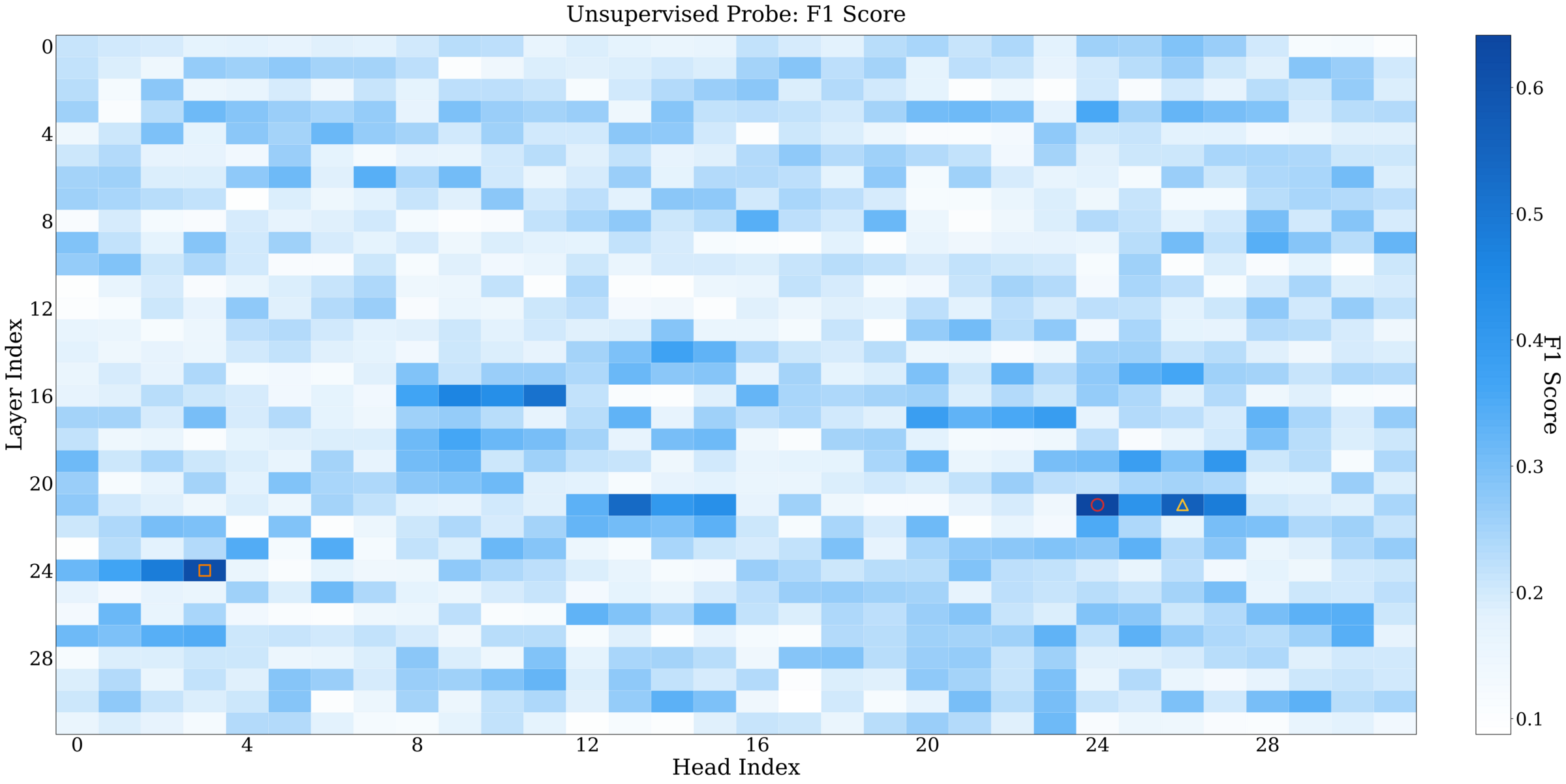}
\caption{Unsupervised head-level classification performance on the Fruits task using SAE latent representations extracted from Llama-3-8B.}
\label{fig:sae_llama8b_fruits_head}
\end{figure}

% SAE -> mistral -> Companies
\begin{figure}[H]
\centering
\begin{tabular}{cc}
    \includegraphics[width=0.52\linewidth]{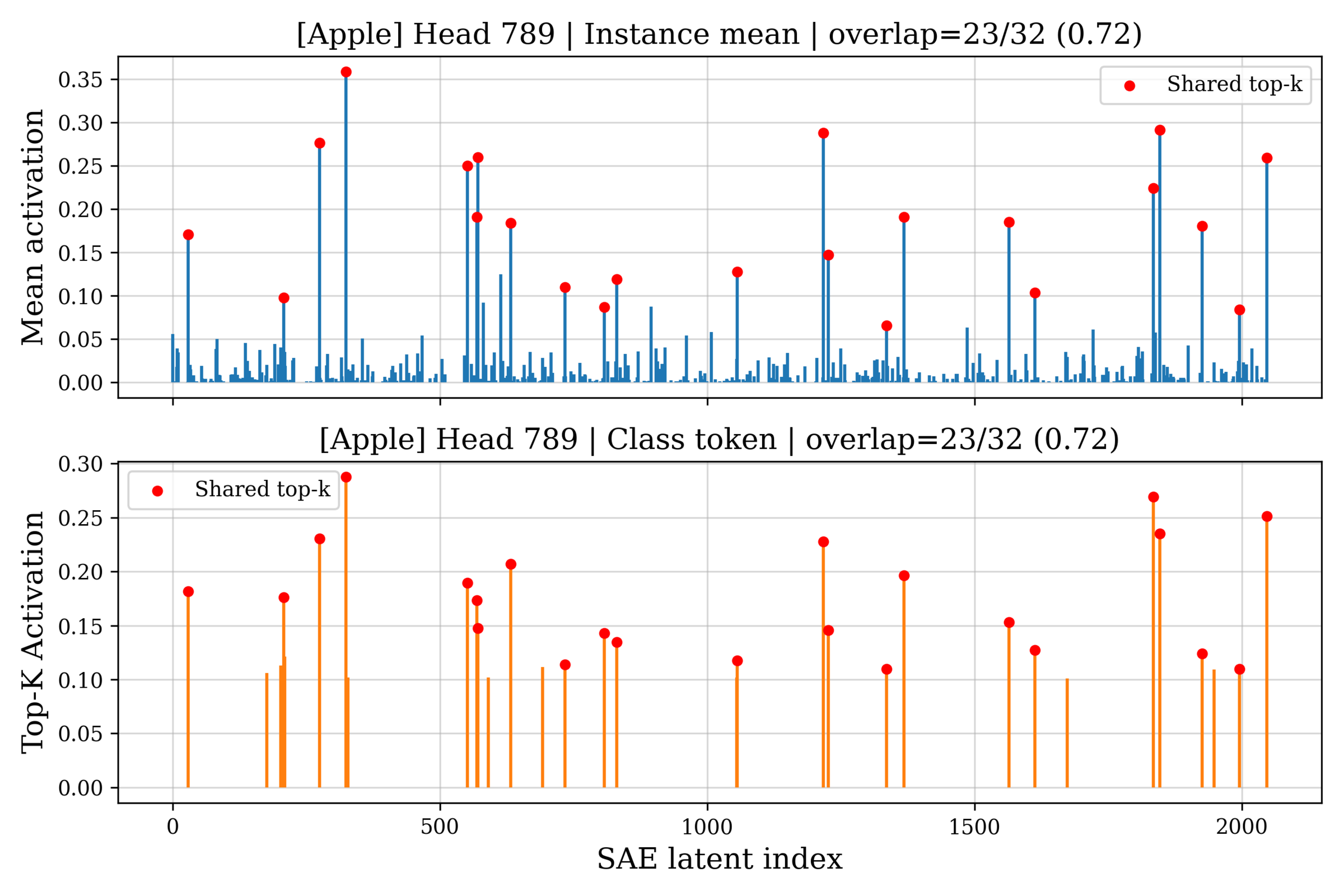} &
    \includegraphics[width=0.52\linewidth]{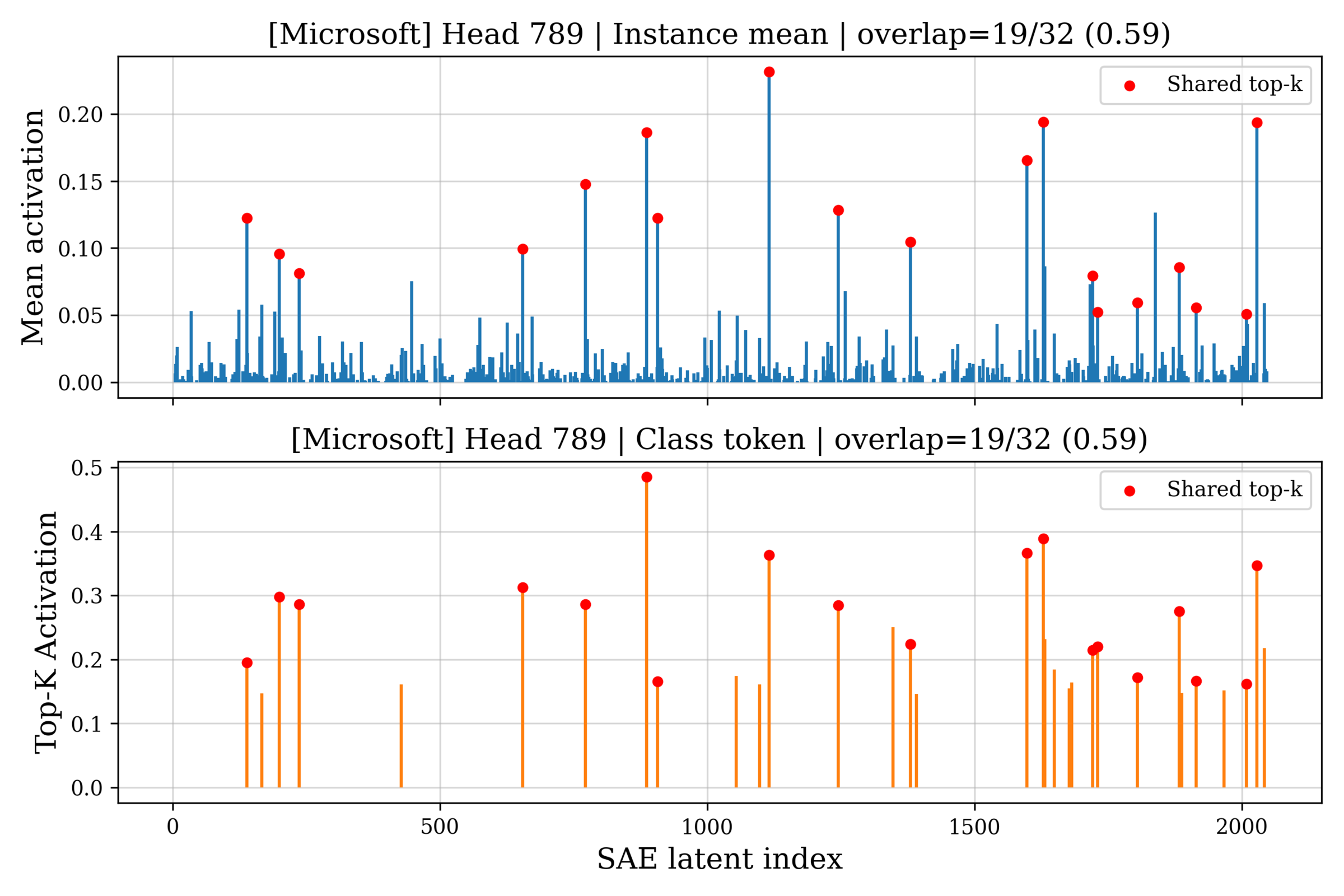}
\end{tabular}
\caption{Shared top-$k$ SAE latent activations between instance-mean and class-prototype representations for the Apple and Microsoft classes in Mistral, indicating head-specific sparse features aligned with company semantics.}
\label{fig:sae_mistral_companies_overlap}
\end{figure}

\begin{figure}[H]
\centering
\includegraphics[width=0.7\columnwidth]{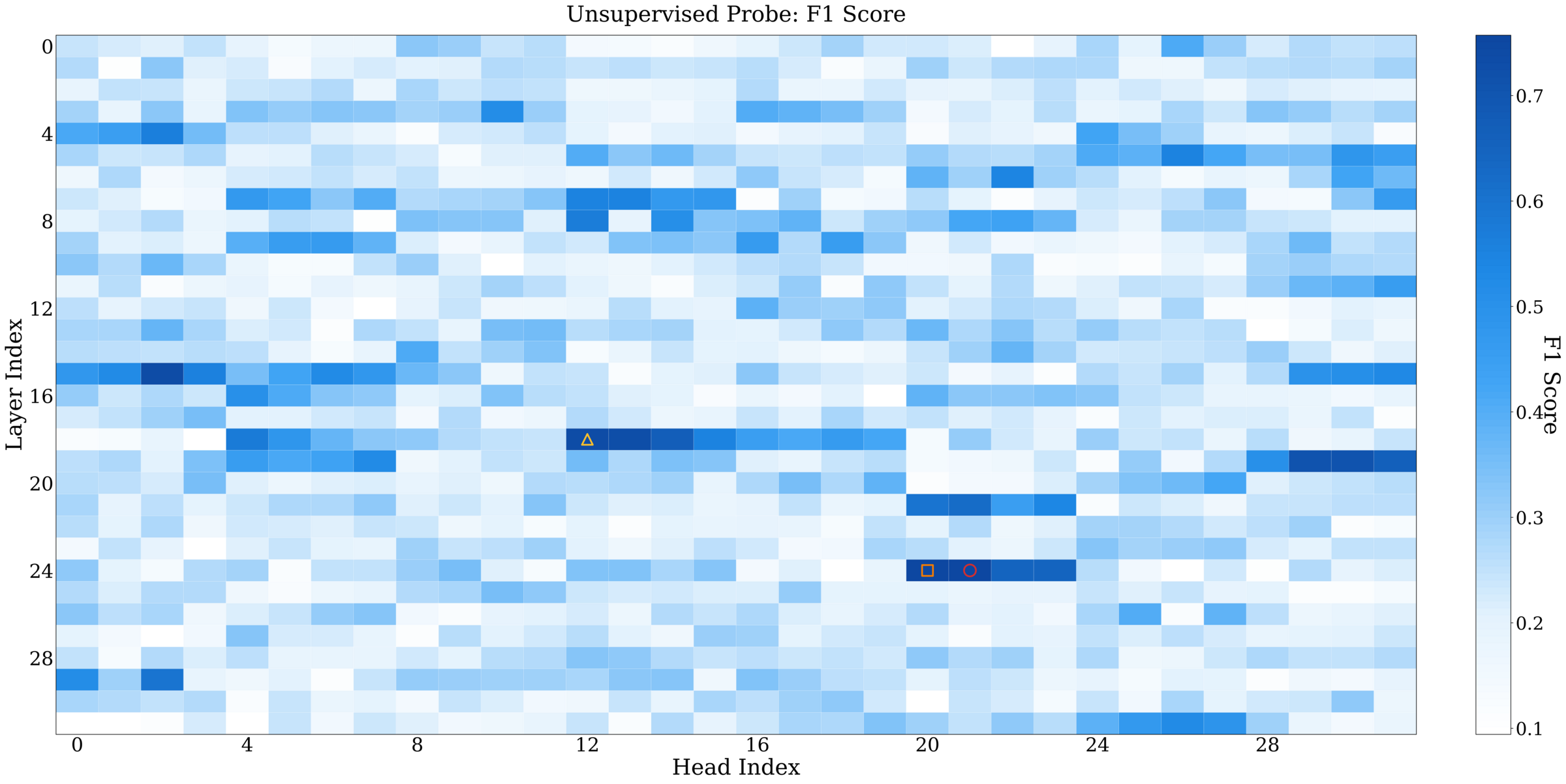}
\caption{Unsupervised head-level classification performance on the Companies task using SAE latent representations extracted from Mistral.}
\label{fig:sae_mistral_companies_head}
\end{figure}

% SAE -> gpt2 -> emotions
\begin{figure}[H]
\centering
\begin{tabular}{cc}
    \includegraphics[width=0.52\linewidth]{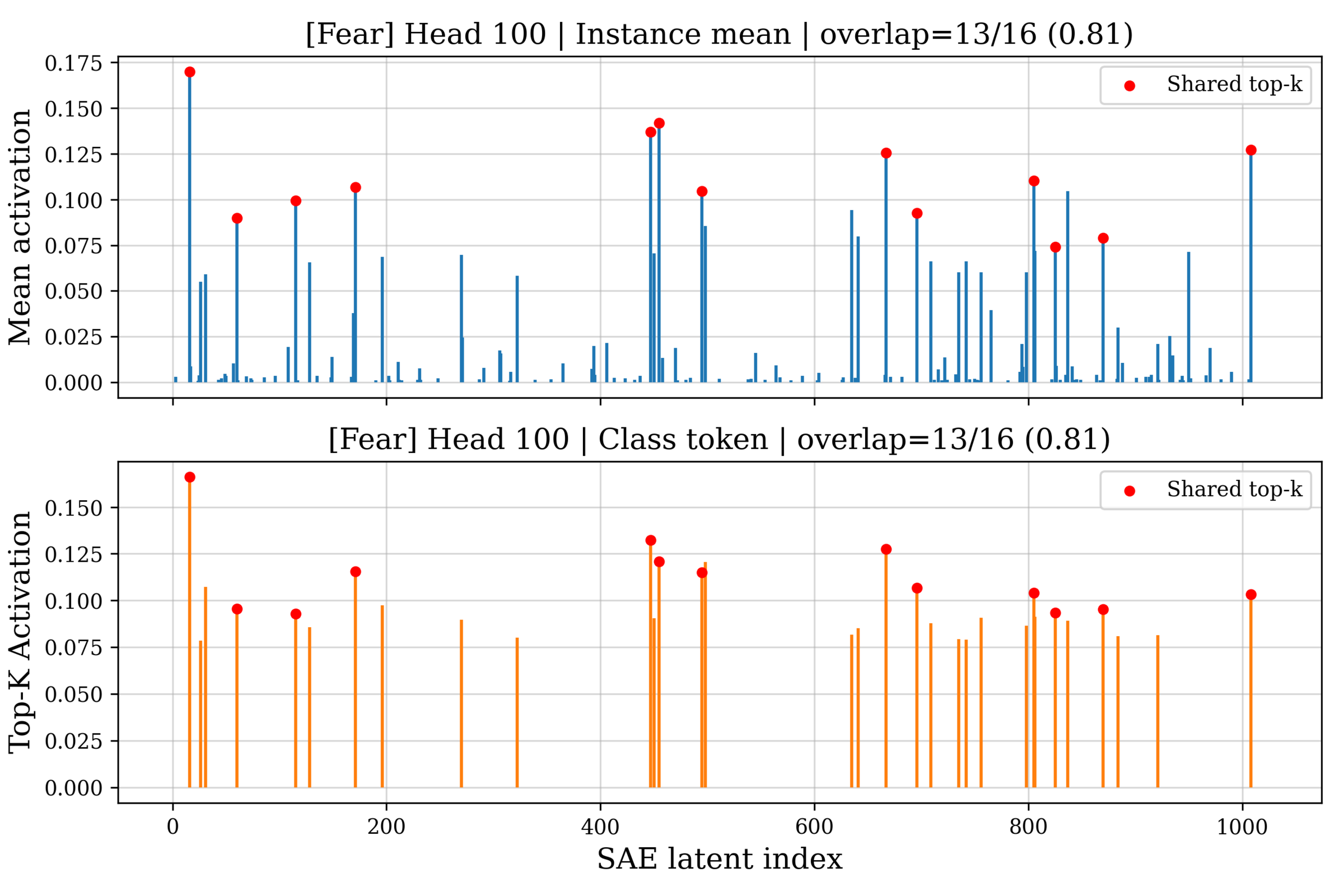} &
    \includegraphics[width=0.52\linewidth]{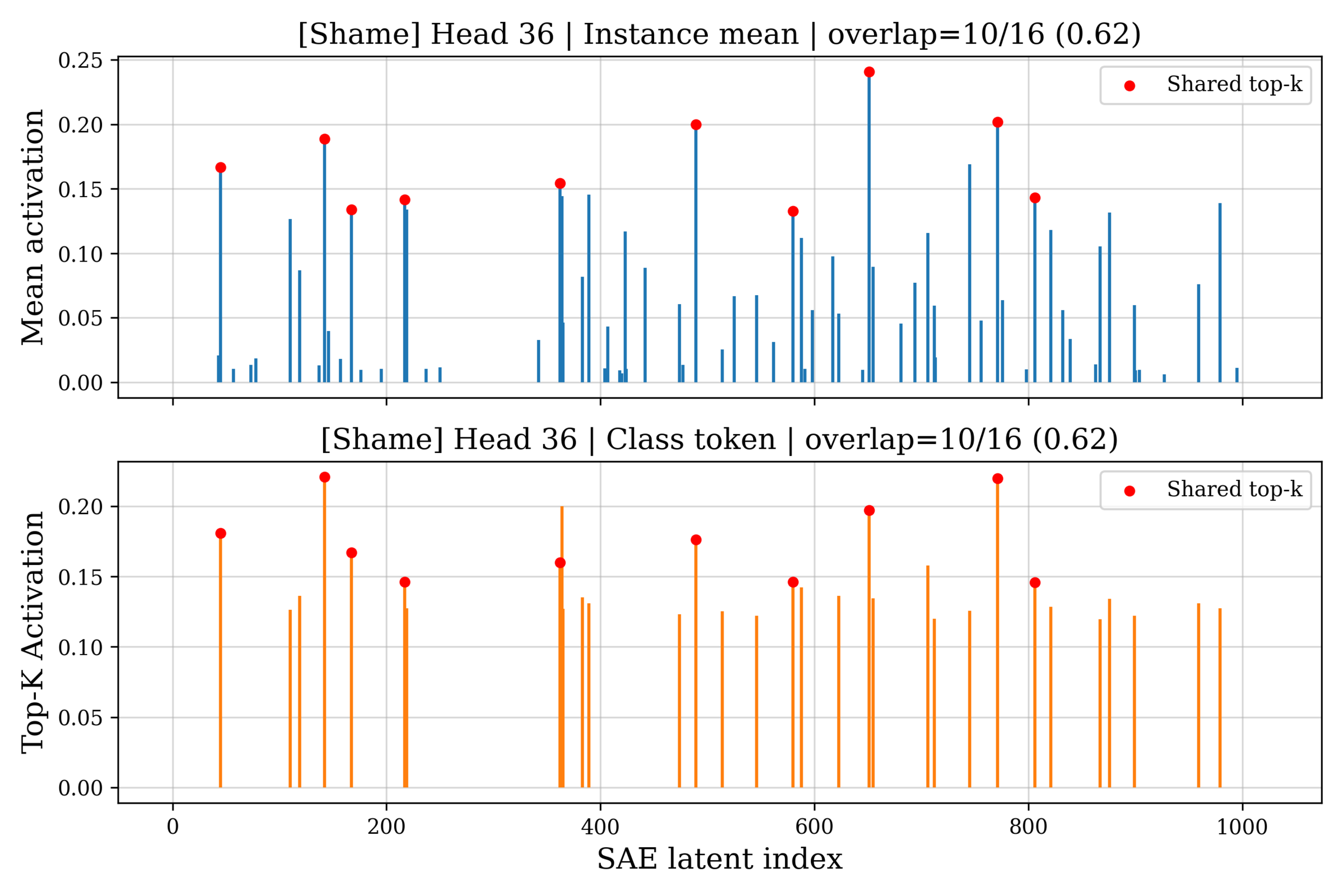}
\end{tabular}
\caption{Shared top-$k$ SAE latent activations between instance-mean and class-prototype representations for the Fear and Shame emotion classes in GPT-2.}
\label{fig:gpt2_emotion_overlap}
\end{figure}

\begin{figure}[H]
\centering
\includegraphics[width=0.7\columnwidth]{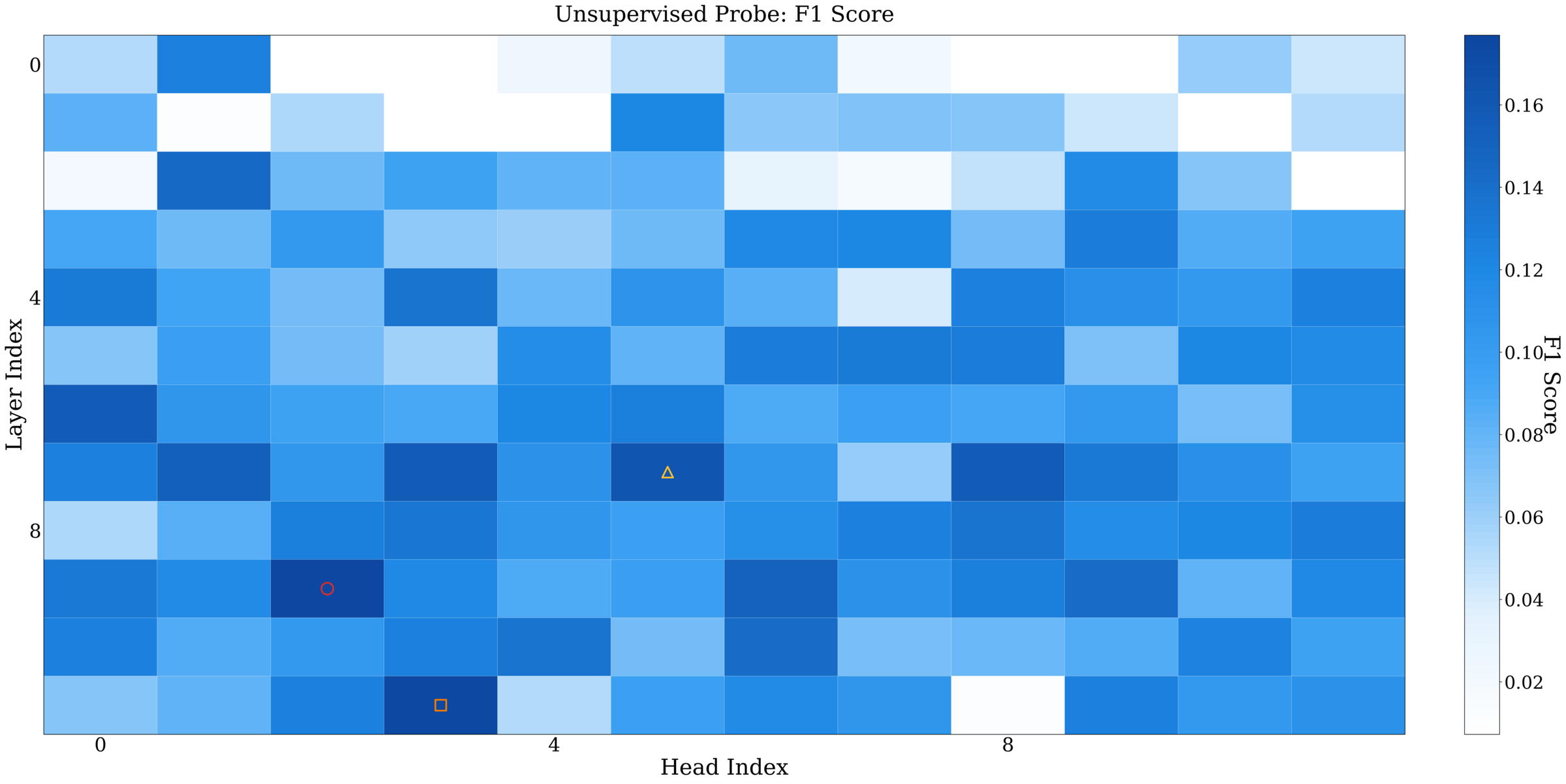}
\caption{Unsupervised head-level classification performance on the Emotions task using SAE latent representations extracted from GPT-2.}
\label{fig:gpt2_emotion_head}
\end{figure}

%%%%%%%%%%%%%%%%%%%%%%%%%%%%%%%%%%%%%%%%%%%%%%%%%%%%%%%%%%%%%%%%%%%%%%%%%%%%%%%
%%%%%%%%%%%%%%%%%%%%%%%%%%%%%%%%%%%%%%%%%%%%%%%%%%%%%%%%%%%%%%%%%%%%%%%%%%%%%%%
\subsection{Dataset Explanation}
\label{app:dataset}

We constructed several datasets to evaluate where semantic attributes are encoded in LLMs and whether semantic features exhibit this linear invariance.

\paragraph{Countries.} We analyze the spaces of all the countries in the world grouped by their respective continent.

\paragraph{Languages.} To evaluate cross-lingual representations, we employed the XQuAD dataset~\citep{artetxe2020cross}, a benchmark for cross-lingual question answering. We sampled 200 questions per language, covering English, Spanish, Russian, Hindi, German, and Mandarin Chinese.

\paragraph{Emotions.} We used a subset of the Emotion Cause dataset~\citep{ghazi2015detecting}, which includes 1,594 English sentences annotated with seven emotion labels (fear, sadness, anger, happiness, surprise, disgust, and shame). We sampled 600 examples for evaluation.

\paragraph{Cartoon Characters.} To probe stylistic and character-specific features, we constructed a dataset centered on fictional characters with distinctive linguistic patterns. For each of six characters---Elmer Fudd, Foghorn Leghorn, Jar Jar Binks, Porky Pig, Scooby-Doo, and Yoda---we collected 50 iconic phrases from publicly available sources.

\paragraph{Literary Authors.} We obtained 20 book quotes per author---William Faulkner, Gabriel Garc\'ia M\'arquez, Ernest Hemingway, Edgar Allan Poe, Virginia Woolf, William Shakespeare, and Mark Twain---from Goodreads.

\paragraph{Animals.} This dataset includes animal species names categorized by biological class: mammals, invertebrates, birds, amphibians, reptiles, and fish.

\paragraph{Fruits and Companies.} To test polysemy, we compiled a descriptive dataset where tokens like ``Apple'' belong to both the fruit and company categories, enabling evaluation of how context disambiguates shared representational directions. The company classes were ``Apple'', ``Google'', ``Microsoft'', and ``Amazon'', while the fruit classes were ``Apple'', ``Grape'', ``Orange'', and ``Banana''..

\end{document}